\documentclass[10pt]{article} 

\usepackage[accepted]{rlj}

\usepackage{amssymb}            
\usepackage{mathtools}          
\usepackage{mathrsfs}           
\usepackage{graphicx}           
\usepackage{subcaption}         
\usepackage[space]{grffile}     
\usepackage{url}                
\usepackage{lipsum}             

\usepackage{amsthm}

\newtheorem{prop}{Proposition}
\newtheorem{lemma}{Lemma}

\renewcommand\cite{\citep}

\newcommand*{\doubleN}{\ensuremath{\mathbb{N}}}
\newcommand*{\doubleE}{\ensuremath{\mathbb{E}}}

\newcommand{\indep}{\perp \!\!\! \perp}

\title{Bayesian Meta-Reinforcement Learning with \\Laplace Variational Recurrent Networks}

\setrunningtitle{Laplace Variational Recurrent Networks}

\author{Joery A. de Vries\textsuperscript{$\dagger$}, Jinke He, Mathijs M. de Weerdt, Matthijs T.J. Spaan}

\emails{\{J.A.deVries, J.He-4, M.M.deWeerdt, M.T.J.Spaan\}@tudelft.nl}

\affiliations{
\textbf{EEMCS, Delft University of Technology, The Netherlands}\\
\par 
$^\dagger$ Corresponding author
}

\contribution{
    We formulate a probabilistic graphical model to match the practical design of memory-based meta-reinforcement learning agents, in order to perform uncertainty quantification through the Laplace approximation \emph{without} retraining or architecture modifications.    
}{
    Ours is an extension of the variational recurrent neural networks by \citet{chung2016recurrent}, \emph{maximum a posteriori} policy optimization by \citet{abdolmaleki2018maximum}, and adopts the control-as-inference framework \cite{levine2018reinforcement}.
}

\contribution{
    We investigate how different assumptions on the posterior model over Markov decision processes interact with representation learning and uncertainty quantification of the recurrent neural network.
}{
    The agents trained with a recurrent neural network are non-Bayesian agents to which we try to apply a Bayesian approximation. Although we obtain a method for quantifying uncertainty in their learned representation, there is still a degree of misspecification.
}

\contribution{
    When used as an alternative to the baseline variational recurrent network, we show that our method either matches or improves performance.
}{
    This shows that our probabilistic formulation provides an alternative approximation for variational online learning while using fewer learnable parameters and without needing architecture modifications.
}

\contribution{
    Our results show that the recurrent neural network representations learned by non-Bayesian meta-reinforcement learning agents, judging over multiple assumptions on the graphical model, produces overconfident estimators.
}{
    This extends prior insight by \citet{xiong_practical_2021} that the representations of memory-based meta-reinforcement learning agents learn \emph{inconsistent} estimators.
}


\summary{Meta-reinforcement learning trains a single reinforcement learning agent on a distribution of tasks to quickly generalize to new tasks outside of the training set at test time. From a Bayesian perspective, one can interpret this as performing amortized variational inference on the posterior distribution over training tasks. Among the various meta-reinforcement learning approaches, a common method is to represent this distribution with a point-estimate using a recurrent neural network. We show how one can augment this point estimate to give full distributions through the Laplace approximation, either at the start of, during, or after learning, without modifying the base model architecture. With our approximation, we are able to estimate distribution statistics (e.g., the entropy) of non-Bayesian agents and observe that point-estimate based methods produce overconfident estimators while not satisfying consistency. Furthermore, when comparing our approach to full-distribution based learning of the task posterior, our method performs on par with variational baselines while having much fewer parameters.
}

\begin{document}

\makeCover  
\maketitle  

\begin{abstract}
 Meta-reinforcement learning trains a single reinforcement learning agent on a distribution of tasks to quickly generalize to new tasks outside of the training set at test time. From a Bayesian perspective, one can interpret this as performing amortized variational inference on the posterior distribution over training tasks. Among the various meta-reinforcement learning approaches, a common method is to represent this distribution with a point-estimate using a recurrent neural network. We show how one can augment this point estimate to give full distributions through the Laplace approximation, either at the start of, during, or after learning, without modifying the base model architecture. With our approximation, we are able to estimate distribution statistics (e.g., the entropy) of non-Bayesian agents and observe that point-estimate based methods produce overconfident estimators while not satisfying consistency. Furthermore, when comparing our approach to full-distribution based learning of the task posterior, our method performs on par with variational baselines while having much fewer parameters.
\end{abstract}


\section{Introduction}
Reinforcement Learning (RL) concerns itself with making optimal decisions from data \cite{sutton_reinforcement_2018}. This is typically achieved by letting an agent generate data in an environment and then optimizing a cost function of the agent's parameters given this data.  In meta-RL, an agent is trained to optimize an expected cost over a prior distribution of environments \citep{finn2017modelagnostic, chen2017learning, beck_survey_2023}. The idea is then that, given a trajectory of new data, an agent can infer latent environment parameters and successfully adapt its action policy online. This is known as zero-shot or few-shot adaptation or learning \citep{beck_survey_2023}. In recent years, this paradigm has shown impressive results, for example, by the Capture the Flag agent \citep{Jaderberg_2019} or the Adaptive Agent \citep{ada_timescale_2023}.

In any meta-RL algorithm, an accurate approximation of the latent parameter distribution given data, also known as the task posterior distribution, is useful to quantify the agent's uncertainty \cite{grant2018recasting}. Accurate quantification of uncertainty enables agents to detect distribution shifts \cite{daxberger_laplace_2021} or guide exploration through novelty signals \cite{osband_psrl_2013, ramanan_plantoexplore_2020}. Importantly, on deployment, distribution shift detection is essential for timely human intervention or retraining. This ultimately improves the robustness and efficacy of our algorithms and allows us to more reliably inspect failure cases.

A common approach in meta-RL is to model the task posterior with point estimates, e.g., using the hidden state of recurrent neural networks (RNN) \citep{chen2017learning}. However, this prevents us from exploiting useful distributional statistics. Another downside of using point-estimates is the increased risk of overconfidence unless the true posterior is sharply peaked at that particular point. This would imply that there exists almost no uncertainty about the environment, which is typically a strong and unrealistic assumption. As a consequence, point-estimate based meta-RL has been known to overfit to its training distribution, leading to brittle downstream performance \cite{xiong_practical_2021, greenberg2023train}.

Arguably, a better approach would be to explicitly parameterize some full distribution (e.g., a Gaussian) \citep{chung2016recurrent, zintgraf_varibad_2020}. However, approximate Bayesian methods are slower to train and are still often outperformed by simple point-estimate methods in terms of expected returns \cite{greenberg2023train} even though they model the posterior more accurately. This could be explained by the fact that non-Bayesian methods enjoy reduced sampling noise, easier numerical representation, and improved model capacity by not having to learn a complex posterior model \cite{goyal2017zforcing, hafner2019learning}. 

To get benefits from Bayesian methods when using non-Bayesian models, we introduce the \textit{Laplace variational recurrent neural network} (Laplace VRNN) which utilizes the Laplace approximation to extend RNN-based meta-reinforcement learning. Our method can perform uncertainty quantification for non-Bayesian meta-RL agents without modifying the model architecture or loss function, and without needing to retrain any parameters. In other words, the consequence of the Laplace approximation is that we can apply it at any point during model training. When applied after training, this is often referred to as a \textit{post-hoc} posterior \cite{daxberger_laplace_2021}. This allows us to make use of deterministic pre-training schedules and benefit from their aforementioned advantages while also enjoying the benefits of Bayesian methods. Although the Laplace approximation has already been explored in meta-RL \cite{grant2018recasting, finn2019probabilistic}, it has not been applied in memory-based methods \cite{duan_rl2_2016, beck_survey_2023} which is what we explore. 

The Laplace approximation is a simple method that only requires the \emph{curvature} of a distribution's log-likelihood at a local maximum \citep{daxberger_laplace_2021}. For a Gaussian mean-field assumption on our posterior model \cite{bishop2007}, we only require the Jacobian matrix of the RNN output with respect to its hidden state. This gives us a Gaussian distribution for the task posterior distribution centered at the RNN hidden state with inverse covariance equal to the sum of Jacobian outer products. This is a comparatively cheap approximation compared to typical methods that apply the Laplace approximation to the much higher dimensional neural network parameters \cite{grant2018recasting, daxberger_laplace_2021, martens2020optimizing}.

We empirically validate that our method can \emph{reliably estimate posterior statistics of our non-Bayesian baselines without degrading performance} on supervised and reinforcement learning domains. Similarly to  \citet{xiong_practical_2021}, our results show that non-Bayesian meta-RL agents do not learn consistent estimators, however, we also find that the learned representations are overconfident. This could be seen by inspecting the \textit{post-hoc} posterior provided by the Laplace approximation, which showed low entropy while not converging to a stable distribution. Furthermore, when comparing our method against variational inference baselines, we find that our Laplace method performs on par in terms of mean returns. Ultimately, this shows that the Laplace approximation can complement (or serve as an alternative to) variational inference methods for uncertainty quantification, since we do not \emph{learn} our local uncertainty but estimate this based on the model's fitted parameters.

\section{Related Work}

\paragraph{Meta-Reinforcement Learning} Meta-learning has been described from various viewpoints, ranging from contexts \citep{blockcontext_sodhani_2022}, latent-variable models \citep{garnelo2018neural, wu2019metaamortized, gordon2019metalearning}, amortized inference \cite{gershman2014amortized, wu2019metaamortized}, and ``learning to learn'' \citep{beck_survey_2023, wang2017learning, hospedales2020metalearning}. Applying these ideas to reinforcement learning has been gaining traction within the field recently, for example, learning maximum likelihood estimation algorithms (like our work) \citep{learn2learn_lee_2016, garnelo2018neural}, probability density functions \citep{lu2022discovered, bechtle2021metalearning}, or model exploration strategies \citep{gupta2018metareinforcement}.

Related to our work is the neural process by \citet{garnelo2018neural} which formalizes using meta-learning to infer a set of (global) latent variables for a generative distribution. Since we test on reinforcement learning problems, one could view our model as a type of non-stationary stochastic process or sequential neural process \citep{oksendal_stochastic_2003, gautam_sequential_2019}. This makes our model and optimization objective similar to the PlaNet model \cite{hafner2019learning}, however, we do not condition our recurrent model on samples from the task posterior so we can obtain an analytical solution. This choice does reduce the non-linearity of our model, the topic of linear vs. non-linear state space models is still active research \cite{gu2023mamba}.

\paragraph{Bayesian Reinforcement Learning} Learning an optimal control policy conditional on a task-posterior amounts to approximating the Bayes-adaptive optimal policy \citep{duff_optimal_2002}. In this framework an agent is conditioned on its current state and a history of observations, the history can then be used to produce a belief distribution over latent variables. The Bayes-adaptive optimal policy maximizes the environment returns in expectation over this belief \citep{duff_optimal_2002, ghavamzadeh_bayesian_2015, zintgraf_varibad_2020, mikulik_metatrained_2020}. Although meta-learning induces uncertainty only over the reward and transition function, many have also successfully tackled the problem as a general partially observable Markov decision process \citep{chen2017learning, ada_timescale_2023}. Doing this is orthogonal to our method, however, we focus on the Bayes-adaptive framework.

\paragraph{Laplace Approximation} The Laplace approximation has been explored for meta-learning in, for example, the model agnostic meta-learning algorithm \citep{finn2017modelagnostic, finn2019probabilistic} to achieve more accurate inference, or for continual learning \cite{Kirkpatrick_2017} as a regularizer for weight-updates. The main obstacle to using the Laplace approximation in practice is the computation of the inverse Hessian \cite{mackay_bayesian_1992}, which alone has quadratic memory scaling in the number of model parameters. For this reason, in Bayesian neural networks, the block-diagonal factorization has become quite popular \citep{martens2020optimizing}, as used by TRPO \citep{schulman2017trust} or second order optimizers \citep{botev2017practical}. Our method bypasses the costly Hessian problem by modeling a distribution on a small subset of the full parameter set. In contrast to doing Bayesian linear regression on the last layer of a neural network, our method can express multimodal distributions.

\section{Preliminaries}
We want to find an optimal policy $\pi$ for a sequential decision-making problem, which we formalize as an episodic Markov decision process \cite{sutton_reinforcement_2018}. We define states $S\in \mathcal{S}$, actions $A \in \mathcal{A}$, and rewards $R \in \mathbb{R}$ as random variables that we sample in sequences. We write $H^i = \{S_t, A_t, R_t\}^T_{t=1}$ to abbreviate the joint random variable of episode $i \in \doubleN$,
\begin{align}
    p(H^i) = \prod_{t=1}^T p(R_t | S_t, A_t) \pi(A_t | S_t) p(S_t | S_{t-1}, A_{t-1}),\nonumber
\end{align}
where $p(S_1 | A_0, S_0) \overset{\Delta}{=} p(S_1)$ is the initial state distribution, $\pi(A_t | S_t)$ is the policy, $p(S_{t+1} | S_t, A_t)$ is the transition model, and $p(R_t | S_t, A_t)$ is the reward model. To avoid confusion, we denote episodes in the \textit{superscript} from $i =1, \dots, n$ and time in the \textit{subscripts} from $t=1, \dots, T$. For convenience, we subsume the common discount factors $\gamma \in [0, 1]$ into the transition probabilities as a global termination probability of $p_{\text{term}} = 1 - \gamma$ \cite{levine2018reinforcement} assuming that the MDP will end in an absorbing state with zero rewards. The objective is to find $\pi^*$ such that $\mathbb{E}_{p(H)} \sum_t R_t$ is maximized.

\subsection{Inference in Meta-RL} In contrast to single-task Reinforcement Learning (RL), in meta-RL we want to find the optimal policy $\pi^*$ to a distribution over different environments. The agent typically does not know which environment it is currently being deployed in and needs to adaptively switch strategies based on online feedback. We assume that the agent can adapt over \emph{multiple} episodes $H^{1:n}$, as opposed to only one episode $H^1$, which is also known as zero-shot or few-shot adaptation \cite{beck_survey_2023}. This approach can be formalized using the concept of global latent variables $Z$. For a fixed $\pi$, each trajectory $H$ that we sample depends on a sampled latent variable $Z \sim p(Z)$, which could be interpreted as a unique identifier of the current environment. Our agent does not directly observe $Z$, but this variable influences the reward and transition models of the environment. 

With full generality, we define the generative process,
\begin{align} 
    \label{eq:meta_sampling_joint}
    p(H^{1:n}, Z^{1:n}) = \prod_{i=1}^n p(H^{i} | Z^{i}) p(Z^{i} | Z^{<i}, H^{<i}),
\end{align}
where the term $p(H^{i} | Z^{i})$ indicates the sampling distribution of the environment under our current model for $Z$, and the term $p(Z^{i} | Z^{<i}, H^{<i})$ denotes the posterior distribution over latent variables $Z$ given all the data we have observed so far. For brevity, we do not expand the sampling distribution $p(H^{i} | Z^{i})$ here (see Appendix~\ref{ap:lowerbound}), however, this expression hides that the posterior model is also updated inter-episodically at every $S^i_t, A^i_t, R^i_t \in H^i_t$. $Z^{i}$. Note that this model is fully general, and is perhaps more common in continual-RL settings \cite{khetarpal_towards_2022}, yet it reflects the model factorization and inference capabilities captured by most memory-based meta-RL methods \cite{duan_rl2_2016}. 
This generality can be both a feature and a downside; the agent can capture broad environment settings, but combined with function approximation it obfuscates what the agent learns and how posterior uncertainty is represented.

\paragraph{Amortized Inference} 
The posterior $ p(Z^{i} | Z^{<i}, H^{<i})$ from Eq.~\eqref{eq:meta_sampling_joint} is usually intractable; a common approach to deal with this is to use variational inference \cite{bishop2007}. This approach represents the posterior with another distribution $q \in \mathcal{Q}$ within some simpler model class, and then chooses $q$ to maximize a lower-bound to the data marginal (see Appendix~\ref{ap:lowerbound}), i.e.,
\begin{align}
    \label{eq:metarl_joint}
    \ln p(H^{1:n}) &\ge \max_{q \in \mathcal{Q}} \mathbb{E}_{q(Z^{1:n} | H^{1:n})} \sum_{i=1}^n \ln p(H^{i} | Z^{i}) 
    - \mathit{KL} \left( q (Z^{\le i} | H^{<i}) \Vert p(Z^{\le i} | H^{<i})\right).
\end{align}
This involves a functional optimization problem to be repeatedly solved at every timestep. Therefore, a more desirable approach is to amortize this with a learned neural network $f_\theta$ that maps past observations and latents directly to a distribution $q \in \mathcal{Q}$ \cite{gershman2014amortized}. In practice, this can be achieved by using a parametric family for $q_\phi$ and using $f_\theta$ to predict the parameters $\phi \in \Phi$. 

To find the parameters $\theta$ that maximize the evidence lower-bound, we can derive an amortized learning objective. If we abbreviate the maximand from Eq.~\eqref{eq:metarl_joint} as $\mathcal{L}(q, H^{1:n})$ and define $f_\theta: \mathcal{H}^{n} \rightarrow \Phi$ (omitting $Z$ for brevity), we always have the inequality
\begin{align} \label{eq:amortization}
    \max_\theta \: & \doubleE_{p(H^{1:n})} \: \mathcal{L}(q_{\phi = f_\theta(H^{1:n})}, H^{1:n})
    \le  \doubleE_{p(H^{1:n})}  \: [ \max_{\phi \in \Phi} \mathcal{L}(q_\phi, H^{1:n})].
\end{align}
This shows that \textbf{1)} with a sufficient function class for $f_\theta$ and optimal parameters $\theta^*$, we obtain equality when $f_{\theta^*}(H^{1:n}) = \arg \max_{\phi\in \Phi} \mathcal{L}(q_\phi, H^{1:n}), \forall H^{1:n} \in \mathcal{H}^n$, and \textbf{2)} $\theta^*$ can be obtained through a straightforward training procedure. The l.h.s. requires us to be able to sample from $p(H^{1:n})$, and that $\mathcal{L}$ is end-to-end differentiable with respect to $\theta$, which enables the use of stochastic gradient methods \cite{kingma2017adam}. 

\paragraph{From Inference to Control} So far, we have mostly discussed how meta-RL agents can perform inference to latent variables of the environment for a fixed policy $\pi$. To define optimal behavior in the generative process we extend the probabilistic model of Eq.~\eqref{eq:meta_sampling_joint} using the control as inference framework \cite{levine2018reinforcement}. This enables us to apply our Bayesian tools directly to our RL-agent in a theoretically sound manner and, for specific design choices, recovers the typical meta-RL training objective \cite{duan_rl2_2016, wang2017learning}.

Control as inference reformulates classical RL as an inference problem by conditioning the distribution over trajectories $p(H^i)$ on a desired outcome $\mathcal{O}$. This outcome $\mathcal{O} \in \{0, 1\}$ is a binary variable indicating whether a trajectory achieves the outcome ($1)$ or not ($0$). The likelihood of this outcome given a trajectory $H^{i}$ follows the exponentiated sum of rewards, $p(\mathcal{O} = 1 | H^{i}) \propto \exp(\sum_t R_t^{i})$. Using Bayes rule, we can then infer the desired policy as $p(H^{i} | \mathcal{O} = 1)$.

Similarly to the latent variable posterior $q_\phi$, we can estimate $p(H^{1:n} | \mathcal{O} = 1)$ through variational inference by defining a lower bound to the log-likelihood of the outcome variable $\ln p(\mathcal{O} = 1 | H^{1:n})$ using a variational policy $q_\pi(A_t^i | S_t^i, Z_t^i)$. Given our choice for the outcome likelihood, this recovers a regularized RL objective \cite{geist_theory_2019} (see the derivation in Appendix~\ref{ap:lowerbound:rl}),
\begin{align}
    \label{eq:rl_lb}
    &\mathcal{L}(q_\phi, q_\pi) = \mathbb{E}_{q_\phi(H^{1:n}, Z^{1:n})} \sum_{i=1}^n \sum_{t=1}^{T_i} R_t^{i} - \mathit{KL}
    \left(
        q_\pi \Vert \pi
    \right) 
    \\
    & \hspace{8em}
    - \mathit{KL}
    \left(
        q_\phi(Z_t^{i} | Z_{<t}^{\le i}, H_{<t}^{\le i}) \Vert p(Z_t^{i} | Z_{<t}^{\le i}, H_{<t}^{\le i} )
    \right) \nonumber
\end{align}
which is an extension of the MPO objective by \citet{abdolmaleki2018maximum} to include the latent-variable posterior and its KL-term. The indexing of the conditional is slightly overloaded for brevity, it indicates that the task posterior is conditioned on all prior data.

Unfortunately, this lower-bound does not give a practical training objective for both inference or amortization as shown in Eq.~\eqref{eq:amortization}, nor does it find the optimal policy given a fixed $\pi$. Therefore, practitioners often include the following design choices.
\begin{enumerate}
    \item The true posterior $p(Z_t^{i} | Z_{<t}^{\le i}, H_{<t}^{\le i})$ is substituted by the previous variational posterior $q(Z_{t-1}^{i} | Z_{<t-1}^{\le i}, H_{<t-1}^{\le i})$, using only a ``true'' prior $p(Z_0)$ at time and episode 0 \cite{hafner2019learning}.
    \item The policy $\pi$ is iteratively updated to the variational optimum $q_\pi$ \cite{abdolmaleki2018maximum}.
    \item KL-penalties are scaled by parameters $\beta_\pi, \beta_q$.
\end{enumerate}
Finally, observe that the amortized objective for Eq.~\eqref{eq:rl_lb} (i.e., substituted into Eq.~\eqref{eq:amortization}) recovers the typical memory-based meta-RL objective when using a point-estimate (Dirac posterior) $q_\phi(Z | \dots) = \delta(\phi - Z)$ and ignoring the KL-penalties \cite{duan_rl2_2016}.

\section{Laplace Variational RNNs} \label{sec:laplace}
We introduce the Laplace variational recurrent neural network (Laplace VRNN) to make a relatively simple approximation to the variational task posterior $q_\theta \approx \hat{q}_\theta$, to be used in the lower-bounds of Eq.~\eqref{eq:metarl_joint} and Eq.~\eqref{eq:rl_lb}, using the Laplace approximation \cite{mackay_bayesian_1992, daxberger_laplace_2021}. This enables the construction of proper distributions over the latent-variables without introducing additional variational parameters. The idea is that we use this to extend point-estimate methods (i.e., base RNNs) to use distributions at any point during training. For exposition, we introduce our approximation starting from a simpler variational distribution $q_\phi(Z_t | H_{<t})$, dropping superscripts. The full derivation is given in Appendix~\ref{ap:laplace:derivation}.

We use the predicted distribution parameters $\phi_t = f_\theta( H_{<t})$ as a helper variable, and interchange the notation $q_{\phi_t}(Z_t | H_{<t}) = q_\theta(Z_t | H_{<t}, \phi_t)$. Most importantly, the statistical amortization by Eq.~\eqref{eq:amortization} allows us to interpret the mapping $f_\theta: H \mapsto \phi$ as a learned summary statistic for the distribution of $Z$; i.e., a \emph{maximum a posteriori} estimate. We assume that $\phi_t$ is computed autoregressively with a recurrent neural network (RNN) such that $\phi_{t+1} = f_\theta(S_{t}, A_{t}, R_{t}; \phi_{t})$.

We then factorize using a mean-field assumption,
\begin{align}
    q_\theta(Z_{t} | H_{<t}, \phi_{t}) 
    &= \frac{1}{q_\theta(Z_{t} | \phi_{t})^{t-2}} \prod_{i=1}^{t-1} q_\theta(Z_{t} | S_i, R_i, A_i, \phi_{t}), \nonumber
    \\
    &= \exp [ (2 - t) \ln q_\theta(Z_{t} | \phi_{t}) +
    \sum_{i=1}^{t-1} \ln q_\theta(Z_{t} | S_i, R_i, A_i, \phi_{t})  ]\nonumber
    \\ 
    &= \exp h_\theta(Z_{t}; H_{<t}, \phi_{t}),
\end{align}
which gives us the target function $h_\theta$ that we wish to approximate (for the first step, see Lemma~\ref{lemma:iid_cond}; Appendix~\ref{ap:factorization}). For a given $\theta$ and data $H_{<t}$, we use the second order Taylor expansion of $h_\theta \approx \hat{h}_\theta$ linearized at $\phi=\phi_t$, we then exponentiate $\hat{h}_\theta$ and renormalize. Assume that $\phi_t$ is \emph{maximum a posteriori} to $q_\theta(Z_t | H_{<t}, \phi_t)$, then we obtain the Laplace approximation,
\begin{align}
    q_\theta(Z_t | H_{<t}) &\approx \frac{\exp \hat{h}_\theta(Z_t; H_{<t}, \phi_t)}{\int \exp \hat{h}_\theta(z; H_{<t}, \phi_t) dz} \nonumber
    \\[8pt]
    &= \mathcal{N}(\phi_t, \nabla^2_\phi \ln q_\theta( Z_t | H_{<t}, \phi)|_{\phi=\phi_t}),
\end{align}
where $\nabla^2_\phi \ln q_\theta$ is the Hessian of our log-posterior. 

To complete our model from Eq.~\eqref{eq:metarl_joint} and Eq.~\eqref{eq:rl_lb}, we can solve the expectation over $\mathbb{E}_{q_\theta(Z_{<t} | H_{<t})}$ for the posterior $q_\theta(Z_t | H_{<t}, Z_{<t})$ at each time-step $t$, by assuming a convolution of Gaussian densities. The result of this convolution is well-known to be another Gaussian with summed parameters \cite{bromiley2003products},
\begin{align}
     \mu_t &= \phi_t + \sum_{i=1}^{t-1} \mu_i, \label{eq:meansum}
     \\[0pt]
    \Lambda_t &= -\nabla^2_\phi \ln q_\theta( Z | H_{<t}, \phi) |_{\phi = \phi_t} + \sum_{i=1}^{t-1} \Lambda_i. \label{eq:covsum}
\end{align}
In practice, we use a smaller window $H_{k:t-1}$ and $Z_{k:t-1}$ inside the conditional for efficiency. In summary, this implements an RNN where we sum the last $k$ hidden states and covariances for the output-Gaussian, and where the covariances are produced by the Hessian of the log-posterior with respect to the hidden state.

The assumption that $\phi_t$ is \textit{maximum a posteriori} is quite strict and assumes equality in the amortization objective Eq.~\eqref{eq:amortization}. For a sub-optimal $\theta$ or insufficient function class $f_\theta$, this can induce a first-order error in the Taylor expansion of the variational log-posterior, which results in a worse approximation. Furthermore, most RNN methods do not sum their hidden states for the predictive model, which mismatches our formulation. It is also not obvious what representations RNN-based meta-RL agents learn and, thus, which model factorization would best suit the agent for construction of our Laplace approximation. We investigate these technicalities and assumptions in the next section.

\paragraph{Special Case} Our method obtains a particularly nice form if we choose $q_\theta(Z_t | S_i, A_i, R_i, \phi_t)$ to be standard Gaussian and use an uninformative prior for $q_\theta(Z_t | \phi_t)$. In that case, the Hessian of the log-posterior $q_\theta(Z_t | H_{<t}, \phi_t)$ becomes a sum of outer products of our RNN state Jacobians w.r.t. $\phi$. Let $x_i=(S_i, A_i, R_i)$, this gives the inverse covariance,
\begin{align}
    \Lambda_t &= \sum_{i=1}^{t-1} (\nabla_\phi f_\theta(x_i; \phi) |_{\phi=\phi_t})(\nabla_\phi f_\theta(x_i; \phi) |_{\phi=\phi_t})^\top,
\end{align}
which is cheap to compute with forward accumulation \cite{jax2018github}, see Prop.~\ref{ap:prop:rnn_outerjac} in Appendix~\ref{ap:laplace:derivation}.

\paragraph{Posterior Predictive} If we now choose a policy $\pi(A_t | S_t, Z_t)$ that is linear in $Z_t$, then our full model would recover a type of Gaussian process \citep{immer2021improving, rasmussen_gaussian_2004}. However, we model this term with another neural network $\pi_\psi(A_t | S_t, Z_t)$ to improve expressiveness. Our policy is then defined by the \textit{posterior predictive} $\pi_\psi(A_t | S_t, H_{<t})$, which we compute using Monte-Carlo,
\begin{align}
    &\int \pi_\psi(A_t | S_t, z_t) q_\theta (z_t | H_{<t}) dz_t
    \approx 
    \hspace{0.2em} 
    \frac{1}{k} \sum^{k}_{i=1} \pi_\psi(A_t | S_t, Z_t=z^{(i)}), \quad z^{(i)} \sim q_\theta (Z_t | H_{<t}),
\end{align}
overloading superscripts to index Monte-Carlo samples. This induces a finite mixture for the policy where $k=1$ corresponds to posterior sampling \cite{osband_psrl_2013}. The parameters $\theta$ and $\psi$ were trained jointly in an end-to-end manner (as defined in the l.h.s. of Eq.\eqref{eq:amortization}).

Interestingly, during training we found output aggregation to train more stably in the loss when $k > 1$. Thus, during training we chose to average the predicted logits of $\pi_\psi$ over samples $z^{(i)}$, or in the continuous case, we averaged over the parameters of a parametric distribution \cite{wang2020striving}, e.g., the mean and variance of a Gaussian (see Appendix~\ref{ap:ensemble} for a discussion). 

\section{Experimental Validation}

\begin{figure*}[t]
    \centering
    \includegraphics[width=\linewidth]{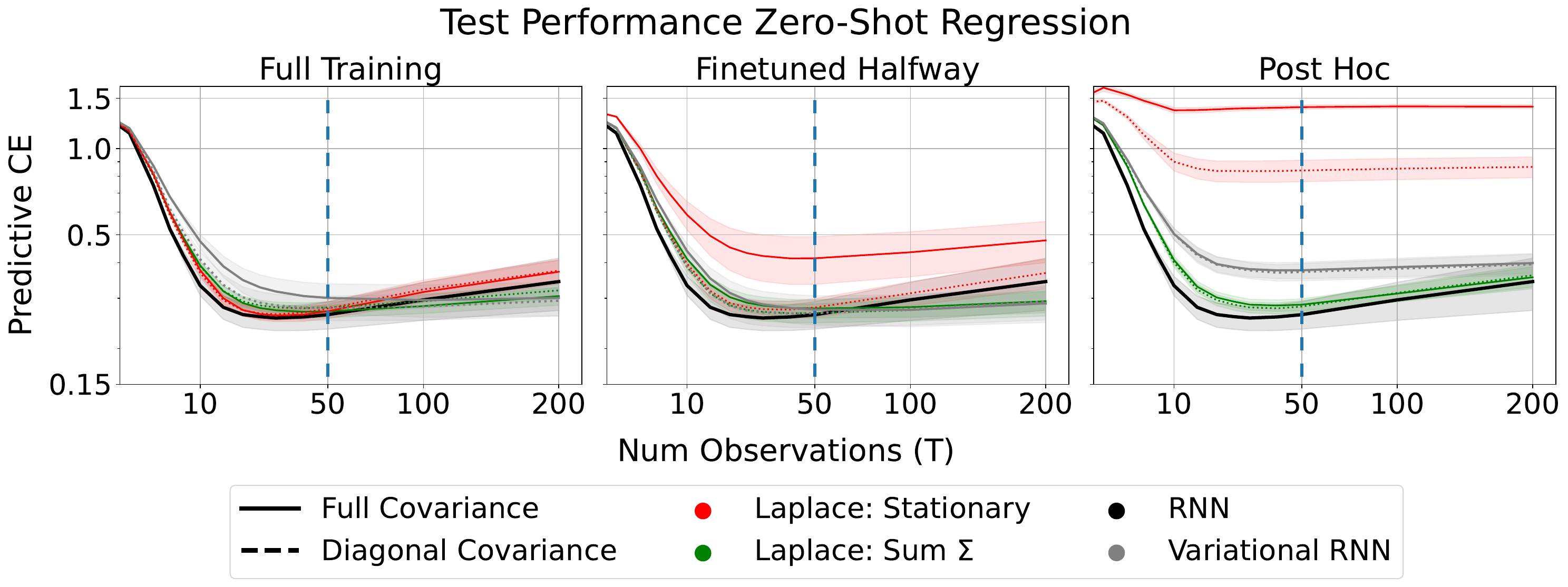}
    \caption{Final performance on the zero-shot regression task in terms of predictive cross-entropy, this should decrease over time. The left column shows results for complete model training, the middle and right columns perform model pre-training with the RNN (black). The middle column includes parameter finetuning, the right column does not. The blue dashed line indicates the training cut-off ($T=50$).}
    \label{fig:supervised_test}
\end{figure*}
In order to apply our method to memory-based meta-RL agents in a manner that didn't alter the network architecture, we required a few extra simplifications to our more general model from Section~\ref{sec:laplace}. 
Crucially, for the \emph{non-stationary} assumption on the log-posterior only, this required us to omit the mean aggregation of Eq.~\eqref{eq:meansum}. 
This wasn't needed for the stationary factorizations, see Appendix~\ref{ap:laplace:final} for a more detailed discussion.
To evaluate our factorizations and design choices, we performed experiments to answer the following:
\begin{enumerate}
    \item \textbf{Utility:} Does our method give useful posterior statistics for a non-Bayesian baseline?
    \item \textbf{Sensitivity:} What model assumptions for the Laplace VRNN are empirically effective?
    \item \textbf{Performance:} When used as an alternative to variational inference, does the Laplace VRNN perform at least on par with existing methods?
\end{enumerate}

Our point-estimate (RNN) baseline was implemented with a long-short term memory architecture \citep{hochreiter_lstm_1997}. The VRNN baseline \cite{chung2016recurrent} extends the RNN by predicting the mean and covariance for a Gaussian distribution as a transformation of the RNN output. For the RNN and VRNN we assumed a stationary factorization of the posterior $q_\theta(Z_t | H_{t-1})$, which is a simplification of the fully general posterior shown in Eq.~\eqref{eq:meta_sampling_joint} and most accurate to the true generative process. We intermittently created model snapshots of the point-estimate baseline (RNN) and finetuned these snapshots over our parameter grid for the Laplace VRNN and VRNN.

All experiments were repeated over $r=30$ seeds (number of network initializations), we tested intermediate model parameters by measuring their in-distribution performance and model statistics for $B=128$ samples (number of test-tasks). We report 2-sided confidence intervals with a confidence level $\alpha=0.99$ for each metric $X$ aggregated over the seeds $r$ and the test-tasks $B$. For the full details on the experiment and baseline setup, see Appendix~\ref{ap:hyperparams}.

\subsection{Supervised Learning} 
As a didactic test-setup, we evaluated our method on noiseless 1D regression tasks. We generated data by sampling parameters to a Fourier expansion and then sampling datasets $\{\{(X_i, f_j(X_i))\}_{i=1}^T\}_{j=1}^n$ where each $X_i \sim \text{Unif}(-1, 1)$, $f_j \sim p_{\text{Fourier}}(f)$, and $n=256, T=50$. 
During training, we optimized a lower bound for a supervised domain using a weight for the KL-term of $\beta=10^{-2}$ (see Eq.~\eqref{eq:metarl_joint}; Appendix~\ref{ap:lowerbound:supervised}). During testing, we computed the predictive cross-entropy (CE) with the true data-generating distribution and our model. So, at each step $t$, we used $H_t = \{X_i, f(X_i)\}_{i=1}^t$ to estimate the posterior predictive distribution $\mathbb{E}_{q_\theta(Z_t | H_t)} p_\theta(Y_i | X_i, Z_t)$ with Monte-Carlo using $m=30$ samples. The predictive CE was estimated using Monte-Carlo over a large test dataset.

\paragraph{Variations} Our Laplace VRNN used a stationary $q_\theta(Z_t | H_{<t})$ assumption (Laplace: Stationary; red) and a Markovian $q_\theta(Z_t | H_{t-1}, Z_{t-1})$ assumption (Laplace: Sum $\Sigma$; green) on the graphical model from Eq.~\eqref{eq:meta_sampling_joint}. In practice, the stationary model computes the covariance for each datapoint in $H_{<t}$ at each $t$, whereas, the Markovian model sums the covariances for each pair $(X_i, Y_i)$. To reduce clutter, we only show the Laplace VRNN ablation that sums the covariance, which also performed best among our variations (c.f., Appendix~\ref{ap:sup_supervised_results}). 
We tested both diagonal and full covariance matrices. 

\begin{figure*}[t!]
    \centering
    \includegraphics[width=\linewidth]{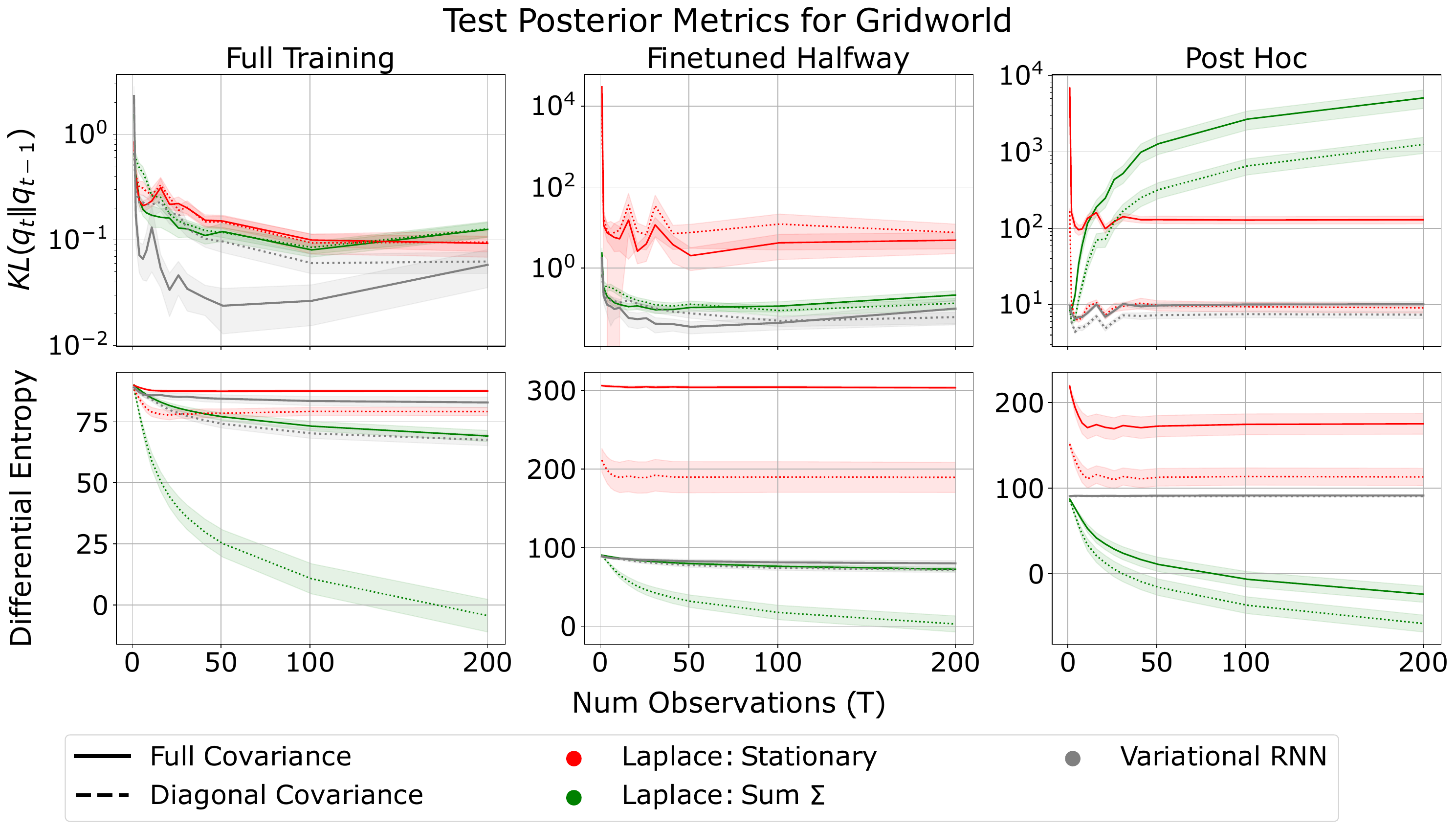}
    \caption{Evolution of summary statistics for the posterior model during testing. The top row shows the KL-divergences between consecutive posteriors $q_t$ and $q_{t+1}$, and the bottom row shows the model entropy over time. In principle, we expect all lines to decrease gradually with more observations. When applying our Laplace approximation with summed covariances (green) after deterministic pre-training (right-column), we see that the posterior becomes more and more confident but does not converge to a stable distribution.}
    \label{fig:grid_test_post}
\end{figure*}
\paragraph{Results} As shown in Figure~\ref{fig:supervised_test}, across our comparisons the predictive CE goes down initially (except for the post-hoc stationary Laplace VRNN), however slightly increases again after the size of the dataset exceeds that seen during training $T > 50$. Notably, we see that our stationary Laplace VRNN strongly degrades performance when not used at the beginning (red), most notably the full-covariance variation. This could be an indication that the Laplace approximated posterior is too wide, whereas the point-estimate is extremely sharp, causing samples from our method to be out-of-distribution for the predictive model. In contrast, this result also shows that our method with the Markovian assumption (green) performs at least as well as the baselines in all cases while also providing a Bayesian posterior for the RNN after training.

\subsection{Reinforcement Learning} 
To show that our method can perform uncertainty quantification while maintaining strong performance in reinforcement learning problems, we evaluated our method on a stochastic 5-armed bandit and a deterministic $5\times 5$ gridworld with sparse rewards. We tested all models using a variant of recurrent PPO \citep{schulman2017proximal} as a simple approximation for Eq.~\eqref{eq:rl_lb}. During training, we used a batch size of $B = 256$ task samples and a sequence length of $T=50$ interactions with the bandit and $T=100$ for the gridworld task. 

For the bandit, we generated training tasks by sampling reward probabilities from a Dirichlet prior using $\vec{\alpha}=0.2$. In this domain we only condition our policy on the sampled model hypotheses $z_t \sim q_\theta (Z_t | H_t)$, $a_t \sim \pi_\theta(Z=z_t)$ as is typical in Thompson sampling \cite{osband_psrl_2013}. This experiment also aimed to investigate robustness to model sampling noise. For the  gridworld \citep{zintgraf_varibad_2020} we sampled tasks by generating the agent's start- and goal-tile uniformly randomly across the grid. In contrast to the bandit problem, the gridworld agent modeled the task as a Bayes-adaptive Markov decision process \citep{duff_optimal_2002}. Meaning that the policy conditions on both the model samples $Z$ and the current state, $a_t \sim \pi_\theta(Z=z_t, S=s_t)$. 

\paragraph{Variations} As before, we test different assumptions for our Laplace VRNN agent's model from Eq.~\eqref{eq:meta_sampling_joint}. In this instance, we used a \emph{windowed} version of the stationary Laplace VRNN $q_\theta(Z_t | H_{t-w-1:t-1})$ for $w=10$ (red). I.e., this truncates the history up to a certain timestep to improve the runtime of the covariance computation, which otherwise scales in $\mathcal{O}(t)$-time. We also tested two variations of the Markovian $q_\theta(Z_t | H_{t-1}, Z_{t-1})$ factorization, which scaled in $\mathcal{O}(1)$-time. The proper-Markovian method (blue) sums the mean and covariance computed at each state-action $(S_t, A_t)$ whereas the second variant only sums the covariance (green).

\begin{figure*}[t!]
    \centering
    \includegraphics[width=\linewidth]{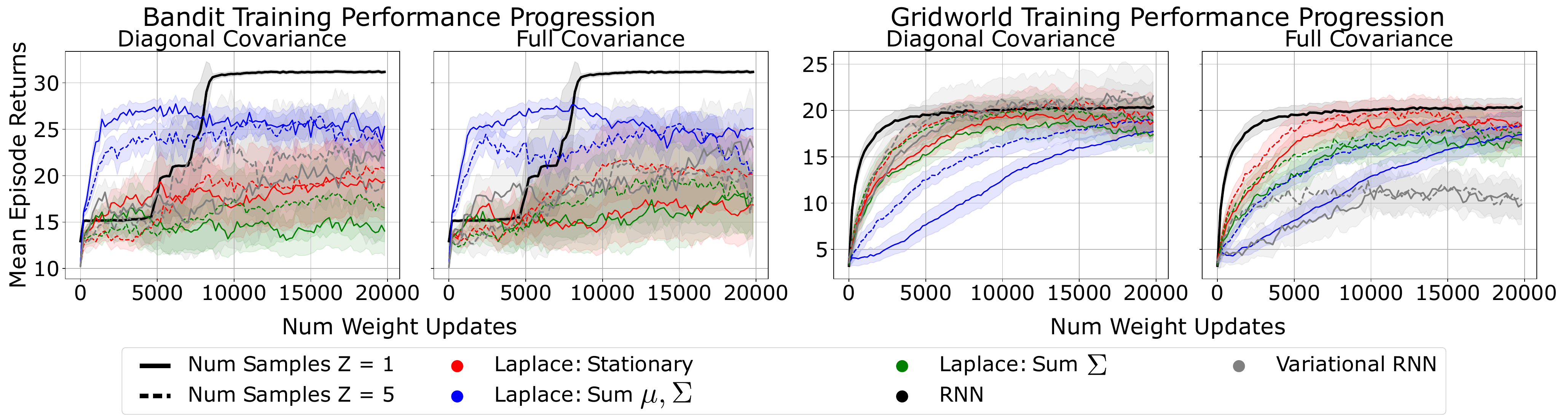}
    \caption{ Average return curves during training for the Reinforcement Learning experiments. The dashed and solid lines (Num $Z$) indicate the number of Monte-Carlo samples used for the posterior model during training inside the modified lower bound of Eq.~\eqref{eq:rl_lb}, to validate off-policy robustness in our loss. As expected, the deterministic RNN performs best, but our Laplace VRNN also outperforms the baseline VRNN. }
    \label{fig:rl_train}
\end{figure*}

\paragraph{Results} We visualize the evolution of estimated posterior statistics during testing in figure~\ref{fig:grid_test_post}, where we removed the ablation that sums both the mean and covariance (blue) to reduce clutter (this ablation performed in between the other two, see Appendix~\ref{ap:sup_grid_results}). We plotted the differential entropy of the posteriors $q_\theta$ and the consecutive KL-divergences $KL(q_\theta(Z_t | \dots) \Vert q_\theta( Z_{t-1} | \dots))$ between posteriors over time, to see whether their behavior matches that of the true posterior. For the true posterior, we expect the entropy to decrease gradually with more observations $T$, which indicates that our model concentrates around some true value. Furthermore, the KL-divergences should converge to zero.

As expected, we see that the posterior entropy of the Bayesian methods reliably goes down and the KL-divergences gets close to zero (left-column). We see a similar pattern when doing finetuning (middle-columns) except for our stationary Laplace variation (red). Most importantly, we see a strong effect of our accumulating covariances variation (green) when using a post-hoc posterior approximation. We see that the entropy steadily decreases while the KL-divergences between consecutive posteriors grows larger and larger. In contrast, the post-hoc VRNN (grey) and stationary Laplace (red) stay relatively constant, and are therefore non-informative. This result shows that the deterministic RNN does not converge to a stable hidden state when not explicitly regularized during training. This means that the learned estimator becomes more and more confident while not being consistent \cite{xiong_practical_2021}.

The average training returns for our model ablations are shown in Figure~\ref{fig:rl_train}. As argued in the introduction, we find that the deterministic method (black) has strong performance while also being the least noisy in the mean episode returns and being the fastest to train in terms of algorithm runtime. Interestingly, the proper-Markovian factorization (blue) of our Laplace VRNN showed faster learning in the Bandit up to a certain point, whereas it degraded training performance for the gridworld. All Bayesian methods tested on the bandit task were significantly noisy and only achieved sub-linear cumulative regret during test-time about $50\%$ over all experiment repetitions. On the grid task, all methods achieved sub-linear cumulative regret.

In summary, none of the ablations for our Bayesian methods degraded performance when applied after deterministic pre-training without finetuning (post-hoc). Our Laplace VRNN also typically performed on par with the VRNN in terms of returns. However, only our Markovian Laplace variation that summed the covariances (green) could produce insightful posterior statistics of the pre-trained model. This confirms all our research questions of whether our proposed Laplace VRNN, and what model assumptions, can give useful posterior statistics while not degrading performance.

\section{Conclusions}
We have described how the Laplace approximation can be applied to recurrent neural network models in a zero-shot meta-reinforcement learning context. Our method is a cheap transformation of an existing recurrent network to a Bayesian model. This enables trained agents to more accurately model their task-uncertainty which can be used to create better or more robust methods.

We tested our method on supervised and reinforcement learning tasks to investigate the utility of our approximation (the quality of posterior statistics), how it depends on model assumptions (ablations), and how it compares against variational inference or point-estimate baselines (no degradation in performance). Our results show that the proposed \emph{Laplace variational recurrent neural network} can reliably transform existing non-Bayesian models to produce a Bayesian posterior, at any point during training without modifying the model or training procedure. In contrast, variational inference requires altering the model architecture and training setup despite matching (or underperforming) compared to our method.

One limitation of our method is the computation of the Jacobians and possible restrictiveness of the Gaussian distribution. Furthermore, the RNN based agents required a variety of simplifications to our probabilistic model formulation (at least, for some of the tested configurations). Although the results matched expected behavior, and is consistent with prior work \cite{xiong_practical_2021, mikulik_metatrained_2020}, future work should investigate what the induced biases entail for our approximated posterior. Our method also does not fix the overconfidence of the non-Bayesian agents, but enables us to observe this effect. Extending our approach to enable statistically sound representation learning or to improve exploration through e.g., distribution-shift detection in meta-RL are a promising directions for further study \cite{daxberger_laplace_2021}.

\section*{Acknowledgements}
This work is supported by the AI4b.io program, a collaboration between TU Delft and dsm-firmenich, and is fully funded by dsm-firmenich and the RVO (Rijksdienst voor Ondernemend Nederland).


\bibliography{references}
\bibliographystyle{rlj}

\beginSupplementaryMaterials


\appendix

\section{Derivations}
\subsection{Lower Bounds} \label{ap:lowerbound}
In this section, we derive the two lower bounds used for model training in the main paper.

\subsubsection{Supervised Learning} \label{ap:lowerbound:supervised}
For the supervised learning domain we can derive an evidence lower bound on the data-marginal as a training objective for our neural network parameters in the following way. For all permutations of $H^{1:n} = \{X^i, Y^i\}_{i=1}^n$, where $X^i \in \mathcal{X}, Y^i \in \mathcal{Y}$, we have,
\begin{align}
    p(H^{1:n}) &= \int p(H^{1:n} | z) p(z) dz
    \\
    &= \int p(X^n, Y^n | z, H^{< n}) p(H^{<n} | z) p(z) dz 
    \\
    & = \int p(X^n, Y^n | z ) p(H^{< n} | z) p(z) dz
    \\
    &= p(H^{< n}) \int p(X^n, Y^n | z ) p(z | H^{< n} ) dz
\end{align}
if we complete the recursion for $p(H^{<n})$ and do importance sampling on the posterior with $q$, we get,
\begin{align}
    \ln p(H^{1:n}) &= \ln  \prod_{i=1}^n  \int p(X^i, Y^i | z^i) p(z^i | H^{1:i}) dz^{1:n}
    \\
    &= \sum_{i=1}^n \ln \int p(X^i, Y^i | z^i) \frac{q(z^i | H^{1:i})}{q(z^i | H^{1:i})} p(z^i | H^{1:i}) dz^{1:n}
    \\
    &\ge \sum_{i=1}^n \int q(z^i | H^{1:i}) \left[ \ln p(X^i, Y^i | z^i) + \ln \frac{p(z^i | H^{1:i})}{q(z^i | H^{1:i})} \right] dz^{1:n}
    \\
    & = \sum_{i=1}^n \mathbb{E}_{q(Z^i | H^{1:i})} \ln p(X^i, Y^i | Z^i) - \mathit{KL}(q(Z^i | H^{1:i}) \Vert p(Z^i | H^{1:i})),
\end{align}
which gives us the lower-bound for our approximate inference model when we use neural network parameters $\theta$ for the predictive and posterior models (overloading notation for $q_\phi$ in Eq.\eqref{eq:amortization}),
\begin{align}
    \mathcal{L}(\theta, H^{1:n}) &= 
    \sum_{i=1}^n 
    \mathbb{E}_{q_\theta(Z^i | H^{1:i})} 
    \underbrace{\ln p_\theta(X^i, Y^i | Z^i)}_{\text{Prediction Loss}}  
    \nonumber    \\
    &\hspace{8em}
    - \beta \cdot
    \underbrace{\mathit{KL}(q_\theta(Z^i | H^{1:i}) \Vert \mathtt{stop\_grad}[ q_\theta(Z^{i-1} | H^{1:i-1}))}_{\text{Complexity Penalty}},
\end{align}
where $\mathtt{stop\_grad}[\cdot]$ indicates a stop-gradient operation and $\mathcal{L}$ should be \emph{maximized} with respect to $\theta$. The hyperparameter $\beta \in \mathbb{R}_+$ accounts for differences in scaling. The stop-gradient is necessary so that the posterior at time $t$ does not depend on the future. During generation and training, we also assume a uniform prior over the inputs $p(X^i | Z) = \text{Unif}$. Of course, this is just one lower bound, the one we use in the main paper for the reinforcement learning tasks also assumes that each $z^i$ is sequentially dependent. In this case, the product would appear inside the integral in the first line for $\ln p(H^{1:n})$, this is only relevant for the Laplace VRNN that accumulates the mean and covariances. 

For simplicity, we only perform training on a single permutation of $H^n$ (i.e., canonical order), as in expectation all permutations are covered anyway and this provides training batches with more diverse examples. Unfortunately, when amortizing the computation of this lower bound with recurrent models it can be difficult to properly distill this permutation invariance of the data into the model. Using a recurrent model that linearly transforms the state, like a transformer \cite{vaswani2017attention, katharopoulos2020transformers} or general state space model \cite{bishop2007}, would prevent this. We leave this open for future work.

\subsubsection{Reinforcement Learning} \label{ap:lowerbound:rl}
Consider the joint distribution over environment traces $H^i = \{S_t, R_t, A_t\}_{t=1}^{T}$ and latent variables $Z$, we'll write episode indices (\textit{extra}-episodic) $i = 1 \dots, n$, in the superscript and time indices (\textit{inter}-episodic) $t = 1, \dots, T$, in the subscript,
\begin{align}
    p(H^{1:n}, Z^{1:n})
    &= 
    \prod_{i=1}^n 
    p(H^{i}, Z^{i} | H^{<i}, Z^{<i})
    \\
    &= \prod_{i=1}^n 
    \prod_{t=1}^{T_i} 
    p(
        S_t^{i}, R_{t}^{i}, A_t^{i}, Z_t^{i} |
        S^{i}_{<t}, R^{i}_{<t}, A^{i}_{<t}, Z^{i}_{<t}, Z^{<i}, H^{<i}
    )
    \\
    &= \prod_{i=1}^n
    \prod_{t=1}^{T_i} 
    p(
        S_t^{i}, R_{t}^{i}, A_t^{i} |
        Z^{i}_{t}, H^{i}_{<t}
    )
    p(
        Z_t^{i} |
        \underbrace{Z^{i}_{<t}, H^{i}_{<t}}_{{inter}},
        \underbrace{Z^{<i}, H^{<i}}_{{extra}}
    )
    \\
    &= 
    \prod_{i=1}^n
    \prod_{t=1}^{T_i} 
    \underbrace{p(
        R_{t}^{i} |
        S_t^{i}, A_t^{i}, Z^{i}_{t}
    )}_{\text{Reward Model}}
    \underbrace{\pi(A_t^{i} | S_t^{i}, Z_t^{i})}_{\text{Action Model}} 
    \nonumber \\
    & \hspace{8em}
    \underbrace{p(S_t^{i} | S_{t-1}^{i}, A_{t-1}^{i}, Z_t^{i})}_{\text{Transition Model}}
    \underbrace{p(
        Z_t^{i} |
        Z^{i}_{<t}, H^{i}_{<t},
        Z^{<i}, H^{<i}
    )}_{\text{Posterior Model}},
\end{align}
the lower-bound in Eq.~\ref{eq:metarl_joint} can then be easily derived by doing importance sampling on the posterior model with $q$, marginalizing out the latent variables, and assuming that $H^{i}$ is independent of all other variables given the latent-variable $Z^{i}$,
\begin{align}
    \ln p(H^{1:n}) &= \ln \int \prod_{i=1}^n 
    p(H^{i}, z^{i} | H^{<i}, z^{<i}) 
    dz^{1:n}    
    \\
    &= \ln \int \prod_{i=1}^n 
    p(H^{i}| z^{i}) \frac{q(z^{1:i} | H^{<i})}{q(z^{1:i} | H^{<i})} p(z^{i} | H^{<i}, z^{<i})
    dz^{1:n}
    \\
    & \ge \int \sum_{i=1}^n q(z^{1:i} | H^{1:i}) \left( \ln p(H^{i}| z^{i}) - \ln \frac{q(z^{1:i} | H^{<i})}{p(z^{i} | H^{<i}, z^{<i})}\right) dz^{1:n}
    \\
    & \propto \mathbb{E}_{q(Z^{1:n} | H^{1:n})} \sum^n_{i=1} \ln p(H^{i} | Z^{i}) - \mathit{KL}(q(Z^{i} | Z^{<i}, H^{<i}) \Vert p(Z^{i} | Z^{<i}, H^{<i})).
\end{align}
As stated in the paper, this lower bound only reproduces the data but does not maximize the rewards per se. So, using the control as inference framework \cite{levine2018reinforcement}, if we write the conditional that a given trajectory $H$ is desirable as $p(\mathcal{O} = 1 | H) \propto \exp (\sum^T_{t=1} R_t)$, then we can derive a lower bound for the sampling distribution for a reinforcement learning agent as (again simplifying notation for $\mathcal{L}$ compared to the main text),
\begin{align}
    \ln p(\mathcal{O} = 1) &= \ln \mathbb{E}_{q(H^{1:n}, Z^{1:n})} p(\mathcal{O} = 1 | H^{1:n}, Z^{1:n}) \frac{p(H^{1:n}, Z^{1:n})}{q(H^{1:n}, Z^{1:n})}
    \\
    &\ge \mathbb{E}_{q(H^{1:n}, Z^{1:n})} \ln p(\mathcal{O} = 1 | H^{1:n})  
 - \mathit{KL}(q(H^{1:n}, Z^{1:n}) \Vert p(H^{1:n}, Z^{1:n})),
 \\
 & = \mathcal{L}(q)
\end{align}
where we define the variational distribution $q(H^{1:n}, Z^{1:n})$ to factorize in exactly the same way as $p(H^{1:n}, Z^{1:n})$ where we fix the reward and transition models and then modify the action and posterior models. This choice of factorization cancels out the fixed terms in the KL-divergence, giving us the lower bound (Eq.\eqref{eq:rl_lb} in the main-text),
\begin{align}
    \mathcal{L}(q) &= \mathbb{E}_{q(H^{1:n}, Z^{1:n})} \sum_{i=1}^n \sum_{t=1}^{T_i} R_t^{i} - \mathit{KL}
    \left(
        q(A_t^{i} | S_{t}^{i}, Z_t^{i}) \Vert \pi(A_t^{i} | S_{t}^{i}, Z_t^{i})
    \right) 
    \\
    & \hspace{8em}
    - \mathit{KL}
    \left(
        q(Z_t^{i} | Z_{<t}^{i}, H_{<t}^{i}, Z^{<i}, H^{<i}) \Vert p(Z_{t}^{i} |  Z_{<t}^{i}, H_{<t}^{i}, Z^{<i}, H^{<i})
    \right).
\end{align}
To amortize computation of this lower bound and make this practical to compute, we parametrize the variational posterior $q_\theta(Z | \dots)$ and action model $\pi_\psi(A | \dots)$. To then finally give us a practical optimization objective for the parameters $\theta, \psi$, we substitute for the action model $\pi(A | \dots) = \pi_{\psi_{old}}$ and for the true posterior we simply use $q_\theta (Z | \dots)$ with a stop-gradient $\square$. We scale the KL-penalty with a hyperparameter $\beta \in \mathbb{R}_+$ to account for differences in scaling. This gives us our final lower-bound,
\begin{align}
    \mathcal{L}(\theta, \psi) &= \mathbb{E}_{q_\theta(H^{1:n}, Z^{1:n})} \sum_{i=1}^n \sum_{t=1}^{T_i} R_t^{i} - \mathit{KL}
    \left(
        \pi_\psi(A_t^{i} | S_{t}^{i}, Z_t^{i}) \Vert \pi_{\psi_{old}}(A_t^{i} | S_{t}^{i}, Z_t^{i})
    \right)  
    \\
    & \hspace{3em}
    - \beta \cdot \mathit{KL}
    \left(
        q_\theta(Z_t^{i} | Z_{<t}^{i}, H_{<t}^{i}, Z^{<i}, H^{<i}) \Vert \square q_{\theta}(Z_{t-1}^{i} |  Z_{<t-1}^{i}, H_{<t-1}^{i}, Z^{<i}, H^{<i})
    \right), \nonumber
\end{align}
which we can plug into the l.h.s. of the amortization objective in Eq.~\eqref{eq:amortization} enabling us to optimize our model parameters through sampling and end-to-end differentiation from the policy to the posterior. Although this is the objective we desire, we make further heuristic approximations through the use of the Proximal Policy Optimization algorithm \cite{schulman2017proximal}. This could roughly be interpreted as doing expectation-propagation \cite{bishop2007} on the policy (i.e., swapping the KL-arguments for the policies). 

Our lower bound is an extension of the one by \citet{abdolmaleki2018maximum}, for standard Markov decision processes, to include the latent variable posterior for use in memory-based meta-reinforcement learning \cite{duan_rl2_2016}. When using an RNN to approximate the posterior, the KL-penalty for $q_\theta(Z | \dots )$ is typically ignored since this is undefined for point-estimates. Doing this would recover the RL$^2$ objective in combination with MPO \cite{abdolmaleki2018maximum, duan_rl2_2016}.

\subsection{Posterior Factorization} \label{ap:factorization}

To choose an efficient factorization for our variational model we need the following result,

\begin{lemma} \label{lemma:iid_cond}
    We can write $p(Z | \{X_i\}^n_{i=1}) = \frac{1}{p(Z)^{n-1}} \prod_{i=1}^n p(Z | X_i)$ iff $X_i \indep X_j, \forall j \ne i$.
\end{lemma}
\begin{proof}
    This result can be shown by applying Bayes rule then factorizing each $X_i$ to be independent of $X_j, \forall j \ne i$ and then applying Bayes rule again,
    \begin{align}
        p(Z | \{X_i\}^n_{i=1}) &= \frac{p(X_1, X_2, \dots, X_n | Z) p(Z)}{p(X_1, X_2, \dots, X_n)} 
        \\
        &= \frac{p(Z)}{\prod_{i=1}^n p(X_i) }
         \prod_{i=1}^n p(X_i | Z) \tag*{(Independence)}
        \\
        &= \frac{p(Z)}{\prod_{i=1}^n p(X_i) }
        \left[ \prod_{i=1}^n p(Z | X_i) \frac{p(X_i)}{p(Z)} \right]
        \\
        &= \frac{p(Z)}{\prod_{i=1}^n p(X_i) }
        \left[ \frac{\prod_{i=1}^n p(X_i)}{p(Z)^{n}} \prod_{i=1}^n p(Z | X_i)  \right]
        \\ &= \frac{1}{p(Z)^{n - 1}}
        \prod_{i=1}^n p(Z | X_i)
    \end{align}
\end{proof}

\subsection{Laplace Variational Recurrent Model}  \label{ap:laplace:derivation}

\begin{prop} \label{prop:laplace}
    Given a mean-field assumption on the data for our posterior $q_\theta$ (Lemma~\ref{lemma:iid_cond}). The second order Taylor Expansion of $\ln q_\theta (Z_t | H_{<t}, \phi_t)$ linearized at $\phi_t$, where $\phi_t=\phi^*$ is a local maximizer of $q_\theta$ and occupies the same space as $Z_t$, yields the following Gaussian distribution,
    \begin{align}
        q_\theta (Z_t| H_{<t}, \phi_t) = \mathcal{N} 
        \left(
            Z_t;
            \mu = \phi_t, 
            \Sigma = (- \nabla^2_\phi \ln q_\theta (Z_t | H_{<t}, \phi) |_{\phi=\phi_t})^{-1}
        \right).
    \end{align}
\end{prop}
\begin{proof}
    Reiterating the results from the main paper, we choose to factorize our model as,
    \begin{align}
        q_\theta(Z_{t} | H_{<t}, \phi_{t}) 
        &= \frac{1}{q_\theta(Z_{t} | \phi_{t})^{t-2}} \prod_{i=1}^{t-1} q_\theta(Z_{t} | S_i, R_i, A_i, \phi_{t}), \tag*{(Lemma~\ref{lemma:iid_cond})}
        \\
        &= \exp \left[ (2 - t) \ln q_\theta(Z_{t} | \phi_{t}) + \sum_{i=1}^{t-1} \ln q_\theta(Z_{t} | S_i, R_i, A_i, \phi_{t})  \right]
        \\
        &= \exp h_\theta(Z_{t}; H_{<t}, \phi_{t}),
    \end{align}
    where our aim is to make a local approximation to $h_\theta$, the rest of the proof follows Appendix~A from \citet{daxberger_laplace_2021}.
    
    The second order Taylor expansion of $h_\theta(Z_{t}; H_{<t}, \phi_{t})$ where $\phi_t$ (locally) maximizes $h_\theta$ keeping all other arguments fixed, gives us,
    \begin{align}
        \hat{h}_\theta(Z_t; H_{<t}, \phi) &= h_\theta(Z_t; H_{<t}, \phi_t) + \underbrace{\nabla_\phi h_\theta |_{\phi=\phi_t} (\phi - \phi_{t})}_{ = \: 0} + \: \frac{1}{2} (\phi - \phi_t)^\top \nabla_\phi^2 h_\theta |_{\phi=\phi_t} (\phi - \phi_{t}),
        \nonumber \\
        &= h_\theta(Z_t; H_{<t}, \phi_t) - \frac{1}{2} (\phi - \phi_t)^\top \nabla_\phi^2 h_\theta |_{\phi=\phi_t} (\phi - \phi_{t}),
    \end{align}
    dropping the function arguments to $h_\theta$ for the higher-order terms for brevity. When exponentiating $\hat{h}_\theta$ and renormalizing it to integrate to $1$, it is easy to show that we recover a Gaussian,
    \begin{align}
        h_\theta(Z_t; H_{<t}, \phi_t) &\approx \frac{1}{\int_\mathbb{R} \exp \hat{h}_\theta(Z_t; H_{<t}, \phi') d\phi' } \exp \hat{h}_\theta(Z_t; H_{<t}, \phi) 
        \\
        &= 
        \frac{
            \exp 
            \{
                h_\theta(Z_t; H_{<t}, \phi_t) - \frac{1}{2} (\phi - \phi_{t})^\top (- \nabla_\phi^2 h_\theta |_{\phi=\phi_t}) (\phi - \phi_t)
            \}
        }
        {
            \int_\mathbb{R} \exp
            \{
                h_\theta(Z_t; H_{<t}, \phi_t) - \frac{1}{2} (\phi' - \phi_t)^\top (- \nabla_\phi^2 h_\theta |_{\phi=\phi_t}) (\phi' - \phi_t) 
            \}
            d\phi'
        }
        \\[8pt]
        &= 
        \frac{
            \exp 
            \{
                -\frac{1}{2} (\phi - \phi_{t})^\top (- \nabla_\phi^2 h_\theta |_{\phi=\phi_t}) (\phi - \phi_t)
            \}
        }
        {
            \int_\mathbb{R} \exp
            \{
                - \frac{1}{2} (\phi' - \phi_t)^\top (- \nabla_\phi^2 h_\theta |_{\phi=\phi_t}) (\phi' - \phi_t) 
            \}
            d\phi'
        }
        \\[8pt]
        &= \mathcal{N}\left(
            Z; \mu = \phi_t, \Sigma = (- \nabla^2_\phi \ln q_\theta (Z_t | H_{<t}, \phi) |_{\phi=\phi_t})^{-1}
        \right).
    \end{align}
\end{proof}
Observe that the above result technically gives us a distribution over $\phi$, and misleadingly not $Z$. However, this is simply a consequence of our notation and distributional assumptions for $q$. In practice, $\phi$ is computed by our representation model (recurrent neural network), and our probabilistic model conflates the learned representation $\phi$ as the mode of the distribution for the latent distribution $Z$.

\begin{prop} \label{ap:prop:rnn_outerjac}
    If we choose $q_\theta(Z_t | S_i, R_i, A_i, \phi) = \mathcal{N}(Z_t; \mu = f_\theta(S_i, R_i, A_i; \phi), \Sigma=I_n)$ and $q_\theta(Z_t | \phi) = \mathcal{N}(Z_t; \mu =\phi, \Sigma = \sigma^2_\phi I_n)$ where we take the limit for $\sigma^2_\phi$ to infinity, then our Laplace approximated posterior (Proposition~\ref{prop:laplace}) has an inverse covariance that is computed as,
    \begin{align}
        \Sigma^{-1}_t = \sum_{i=1}^{t-1} (\nabla_\phi f_\theta(S_i, R_i, A_i; \phi) |_{\phi=\phi_t}) (\nabla_\phi f_\theta(S_i, R_i, A_i; \phi) |_{\phi=\phi_t})^\top.
    \end{align} 
\end{prop}
\begin{proof}
    To see this we only need to write down the Hessian under a local maximum assumption of $\phi=\phi^*$ (Laplace approximation) and substitute the chosen Gaussian distributions in for all terms. 
    \begin{align}
        \nabla^2_\phi \ln q_\theta(Z_t | H_{<t}, \phi)
     &= \nabla^2_\phi \left[ (2 - t) \ln q_\theta (Z_t | \phi) + \sum_{i=1}^{t-1} \ln q_\theta (Z_t | S_i, R_i, A_i, \phi) \right]
     \\
     &= \nabla^2_\phi \left[ (2 - t) \ln \mathcal{N}(Z_t ; \phi, \sigma^2_\phi I_n) + \sum_{i=1}^{t-1} \ln \mathcal{N}(Z_t; f_\theta(S_i, R_i, A_i; \phi), I_n) \right]
     \\
     &= 
     \frac{t-2}{\sigma^2_\phi} I_n + 
     \sum_{i=1}^{t-1} (J_\phi)_i 
     \underbrace{\left( \nabla^2_\mu \ln \mathcal{N}(Z_t; (f_\theta)_i, I_n) \right)}_{= -1} 
     (J_\phi)_i^\top
     \nonumber \\
     & \hspace{6em} + 
     \underbrace{\sum_{i=1}^{t-1}
     \nabla^2_\phi (f_\theta)_i |_{\phi=\phi_t}
     \nabla_\mu \ln \mathcal{N}(Z_t; (f_\theta)_i, I_n)}_{=0, \text{ when } \phi_t = \phi^* \text{ (Laplace approx.)}}
     \\[0pt]
     &= \frac{t-2}{\sigma^2_\phi} I_n - \sum_{i=1}^{t-1} (J_\phi)_i (J_\phi)_i^\top,
    \end{align}
    where we abbreviate $(J_\phi)_i = \nabla_\phi f_\theta(S_i, R_i, A_i; \phi)$ and $(f_\theta)_i = f_\theta(S_i, R_i, A_i; \phi)$. Then, in the case of using an infinite variance Gaussian for the prior, $\lim_{\sigma^2_\phi \rightarrow \infty} \nabla^2_\phi \ln q_\theta(Z_t | H_{<t}, \phi)$, we get,
    \begin{align}
        \Sigma_t^{-1} &= 
        \left(
            -\nabla^2_\phi \ln q_\theta(Z_t | H_{<t}, \phi)|_{\phi=\phi_t}
        \right)
        = 
        \sum_{i=1}^{t-1} (J_\phi)_i (J_\phi)_i^\top
        \\
        &= 
        \sum_{i=1}^{t-1} (\nabla_\phi f_\theta(S_i, R_i, A_i; \phi) |_{\phi=\phi_t}) (\nabla_\phi f_\theta(S_i, R_i, A_i; \phi) |_{\phi=\phi_t})^\top.
    \end{align}
\end{proof}

\subsubsection{Final Model} \label{ap:laplace:final}
To complete our fully general Laplace approximated variational recurrent neural network, we need to define the posterior over \emph{all} previous latent variables, $q_\theta(Z_t | Z_{<t}, H_{<t})$. We can easily plug this dependency in for our Laplace approximation from Prop.~\ref{prop:laplace} by assuming that each consecutive posterior has an additive effect on all future posteriors (i.e., a Gaussian convolution),
\begin{align}
    q_\theta(Z_t | Z_{<t}, H_{<t}) = \mathcal{N}
    \left(
        Z_t; 
        \mu_t = \phi_t + \sum_{i=1}^{t-1} Z_i, 
        \Lambda_t = -\nabla_\phi^2 \ln q_\theta (Z_t | Z_{<t}, H_{<t}, \phi_t)|_{\phi=\phi_t})
    \right).
\end{align}
As long as we do not condition $\phi_t$ on $Z_{<t}$, the dependency on past latent variables becomes a constant w.r.t. $\nabla^2_\phi$, making the covariance independent of these terms. This is in contrast to a recurrent state-space model architecture which does condition on these values \cite{hafner2020dream}, however, our choice permits an analytical solution. It is a known result that the expected posterior then becomes another Gaussian with the means and inverse covariances summed \cite{bromiley2003products},
\begin{align}
    &\phantom{=} \mathbb{E}_{q_\theta(Z_{<t} | H_{<t})} q_\theta (Z_t | Z_{<t},  H_{<t}) \nonumber \\
    &= 
    \mathcal{N} \left(
        \mu_t=\phi_t + \sum_{i=1}^{t-1} \mu_i, \Lambda_t= -\nabla^2_\phi \ln q_\theta( Z | H_{<t}, \phi) |_{\phi = \phi_t} + \sum_{i=1}^{t-1} \Lambda_i
    \right).
\end{align}
 This particular form has also been described by \citet{ritter2018online} in the context of continual learning. It is easy to accumulate the mean and covariances terms sequentially over $t$. Depending on the assumptions one makes on the data-generating distribution, one can sum over fewer terms to make the calculation more efficient. The ones we ran experiments for in the main paper include:
 \begin{enumerate}
     \item \textbf{Stationary Posterior}: $q_\theta(Z_t | H_{<t})$. Full summation over $H_{<t}$ in the calculation of $\Lambda_t$. No summation over previous posteriors $Z_{<t}$.
     \item \textbf{Markov Chain Posterior}: $q_\theta(Z_t | H_{t-1}, Z_{t-1})$. The inverse covariance is calculated only using the most recent observation $t-1$. Only the previous posterior mean and precision are summed with the current mean and precision.
     \item \textbf{Windowed Markov Chain Posterior}: $q_\theta(Z_t | H_{k:t-1}, Z_{t-1})$. The inverse covariance is calculated using the $k$ most recent observations. Only the previous posterior mean and precision are summed with the current mean and precision.
 \end{enumerate}
However, these are all simplified models, whereas the variational recurrent model \cite{chung2016recurrent} we discuss in the main paper is fully general.

\paragraph{On summation of the means} The current formulation for the probabilistic model mismatches with typical recurrent neural network (RNN) architectures. Typically, RNN models do not aggregate their past hidden states for the outputs. However, as discussed in the main paper, we strictly required this simplification for the Markov chain (non-stationary) factorizations, since this was the only approach that would leave the base RNN architecture untouched. Despite that, we can make two heuristic arguments that can partially explain the effect of summing all the previous covariances but omit summation of the mean: 
\begin{itemize}
    \item \textbf{Representation Learning:} When using covariance summation throughout training (no post-hoc posterior, or finetuning of deterministic baselines), an RNN can learn to represent the hidden state aggregation implicitly within the state-update. 
    \item \textbf{Exponential Tilting:} Omitting the mean summation can be interpreted as a form of exponential tilting of Gaussian distributions. This is an importance-sampling technique for rare-event simulation \cite{asmussen_stochastic_2007}. For the correctly chosen tilting parameter, this can have the same effect as subtracting all previous means (at the cost of some bias).
\end{itemize}

\renewcommand{\arraystretch}{1.5}
\section{Implementation Details}  \label{ap:hyperparams}

Our code is available at \url{https://github.com/joeryjoery/laplace-vrnn}.

\subsection{Model Architecture and Optimization} \label{ap:model}
Following the main text we can define our model according to the following components,
\begin{alignat*}{3}
    &\text{Embedding} \hspace{4em} && S^g_t, A_t^g, R_t^g &&= g_\theta(S_t), h_\theta(A_t), w_\theta(R_t),
    \\
    &\text{Recurrent Model} \hspace{4em} && \phi_{t+1} &&= f_\theta(S^g_t, A^g_t, R^g_t; \phi_{t}),
    \\
    &\text{Posterior Model} \hspace{4em} && Z_t &&\sim q_\theta(Z_t | H_{<t}, \phi_t),
    \\
    &\text{Reward Model} \hspace{4em} && \hat{R} &&\sim p_\theta(\hat{R} | Z_t, A^g, S_t^g),
    \\
    &\text{Action Model} \hspace{4em} && A_t &&\sim \pi_\theta(A_t | Z_t, S_t^g),
\end{alignat*}
note that we do not train an action model for the supervised experiments, and we do not use the reward model in the reinforcement learning experiments (i.e., we do not use it to select actions or to do planning). In the supervised case, the reward model simply learns a direct function prediction $\hat{R}$ where the state can be considered stationary. For brevity, we denoted the full set of parameters as $\theta$, in practice each component has its parameters but we jointly optimize for these using end-to-end differentiation. 
 
For the embedding model we used a multi-layered perceptron (MLP) \citep{cybenko_approximation_1989} of two hidden layers of width 256 nodes, we used leaky-ReLU for the activation. The action and predictive model used a three hidden layer MLP with sizes (256, 256, 64), also with leaky-ReLU. For the recurrent model, we use a long-short-term memory module \cite{hochreiter_lstm_1997} with $n=128$ hidden nodes. We projected the outputs (not the carried state) of the LSTM to a smaller $n=64$ feature output vector with a learned affine transform (i.e., a 1-layer MLP without an activation). 

For discrete environments, the action model predicted the $n$ logits for the full action space. For the regression task, the reward model outputs a Gaussian with a learned mean and input-independent variance. 

As noted in the main paper, we apply stop-gradients on the prior term appearing in the KL-divergences (Eq.~\ref{eq:metarl_joint} and Eq.~\ref{eq:rl_lb}). Doing this prevents past posteriors from fitting to data beyond their respective timestep, while still constraining our current posterior to not deviate too much from these terms. We also apply stop-gradients to the \emph{past} mean and covariance terms when \emph{accumulating} these terms. This is only relevant for the Laplace VRNN ablations. The reasoning for this is the same as for the stop-gradient in the KL term, we want to use the past means and covariances as constants at timestep $t$, and not as another learnable parameter. As a side note, we also found that doing this sped up training orders of magnitude.

We developed everything discussed in this paper in Jax v.0.4.23 \cite{jax2018github}. For neural network design, we used the Flax library v0.8.1 \cite{deepmind2020jax}. For optimization, we used the Adamw optimizer \cite{loshchilov2019decoupled} implemented in Optax with learning-rate $= 10^{-3}$, weight-decay = $10^{-6}$,  and the rest on default settings at version v0.1.7. We used gradient-clipping to have a max global norm of 1.0 and a maximum individual gradient of [-5.0, 5.0]. We used our implementation for the PPO algorithm with help from the RLax library to compute the generalized advantage estimators.

\subsection{Variational Posterior Baseline} \label{ap:vrnn}
As discussed in the main paper, typically in meta-reinforcement learning we see that the posterior $q_\theta(Z_t | H_{<t})$ is modeled with a point-estimate and we propose to use the Laplace approximation to convert this point-estimate into a Gaussian distribution. Instead of computing the covariance matrix using the Laplace approximation, we can also directly predict this covariance with another neural network.

So, our variational recurrent neural network (VRNN) baseline predicted the $m(m-1)$ lower triangular elements of the $m\times m$ dimensional covariance matrix (or just $m$ for the diagonal ablations) and the $m$ dimensional mean. We opted for a spectral decomposition $\Sigma = USV$ where $S$ is diagonal and $U$ was predicted through the Cayley map starting from a skew-symmetric matrix. First, we compute $\phi_t$ with our RNN, then we project $\phi_t$ to $\mu_t, S_t, L_t$ with a linear layer such that $\mu_t \in \mathbb{R}^m$,  $S_t \in \mathbb{R}^{m \times m}$ is diagonal and $L_t \in \mathbb{R}^{m \times m}$ is lower-triangular, we then compute $U_t = (I - A_t)(I + A_t)^{-1}$ where $A_t= L_t - L_t^\top$ to get an orthogonal eigenbasis. We then construct the Gaussian as,
\begin{align}
    q_\theta(Z_t | H_{<t}, \phi_t) = \mathcal{N}(Z_t; \mu = \mu_t, \Sigma = U_t (\exp{S_t}) U_t^\top ),
\end{align}
this representation also made matrix inversion incredibly easy as $(U e^S V)^{-1} = U e^{-S} V$ since $U=V^\top$ are orthogonal.

The reason for using the spectral decomposition was that we did not achieve stable training using any variant of the Cholesky factorization for the covariance matrix ($LU$ or $LDU$ decomposition), even if explicitly transforming the eigenvalues to be positive. We suspect that the spectral parameterization trained more stably than the LU or LDL parameterization as the basis matrices $U$ or $V$ are constrained to the group of orthogonal matrices. Therefore, each element in the predicted parameters $L_n$ is also constrained. In contrast, the LDL parameterization leaves the triangular parameters in an unconstrained representation which might enable an unstable representation. However, we did not investigate this problem beyond what was needed to get our method working. 

We also did not accumulate the mean or covariances for the VRNN, unlike for the Laplace VRNN, as this model parameterization could simply learn this function instead (or learn to undo this parameterization). The point of our experiments was not to squeeze performance out of our baseline but to have a strong reference for comparison. We also found that the VRNN was highly competitive with the Laplace VRNN.

\subsection{Predictive Ensemble Averaging} \label{ap:ensemble}
Since we sample latent variables $Z$ from our posterior $q_\theta$, when we pass multiple samples through our predictive model or the action model, we obtain multiple distributions for the output modalities. Typically this induces a mixture distribution, however, we simply averaged out either the logits or the means and variances \cite{wang2020striving}. This can be interpreted as a normalized Gaussian convolution over the output parameters or simply as a kind of bagging strategy over ensemble members $Z=z^{(i)}, i=1, \dots, k$.

The reasoning for opting for ensemble averaging instead of mixture distributions is that this simply achieved more stable training in combination with PPO \cite{schulman2017proximal}. The main problem was caused by the penalty term of the policy entropy, our implementation did not achieve stable training when approximating this term in any way (so neither with mixtures nor with singular distributions), we only achieved stable training when the entropy term could be computed \emph{exactly}. 

We did not investigate in depth why an approximate entropy loss caused training divergence, but we speculate this is due to the problems of training on generated data as discussed in the context of language models by \citet{shumailov2023curse}. In this case, the tails of the action distribution slowly shrink over time as Monte-Carlo estimation of the entropy might not cover low-likelihood events sufficiently with small sample sizes. This phenomenon is referred to as \textit{model collapse}, not to be confused with \textit{posterior collapse} \citep{goyal2017zforcing}. We suspect that this problem could be reduced by using a proper (approximate) Bayesian inference algorithm, like maximum a posteriori optimization \cite{abdolmaleki2018maximum}, which essentially removes the instability of reinforcement learning losses by casting policy learning as a supervised learning problem.

\subsection{Environment Design}
For our testing environments we implemented three problem domains, a $1$D function regression problem \cite{finn2017modelagnostic}, a discrete $n$-armed bandit problem \cite{duan_rl2_2016}, and an $n \times n$ open grid problem \cite{zintgraf_varibad_2020}, as shown in Figure~\ref{fig:test_env}. We generated many variations of these environments to learn from by sampling their dependent task parameters.

\begin{figure*}[h]
    \centering
    \includegraphics[width=\linewidth]{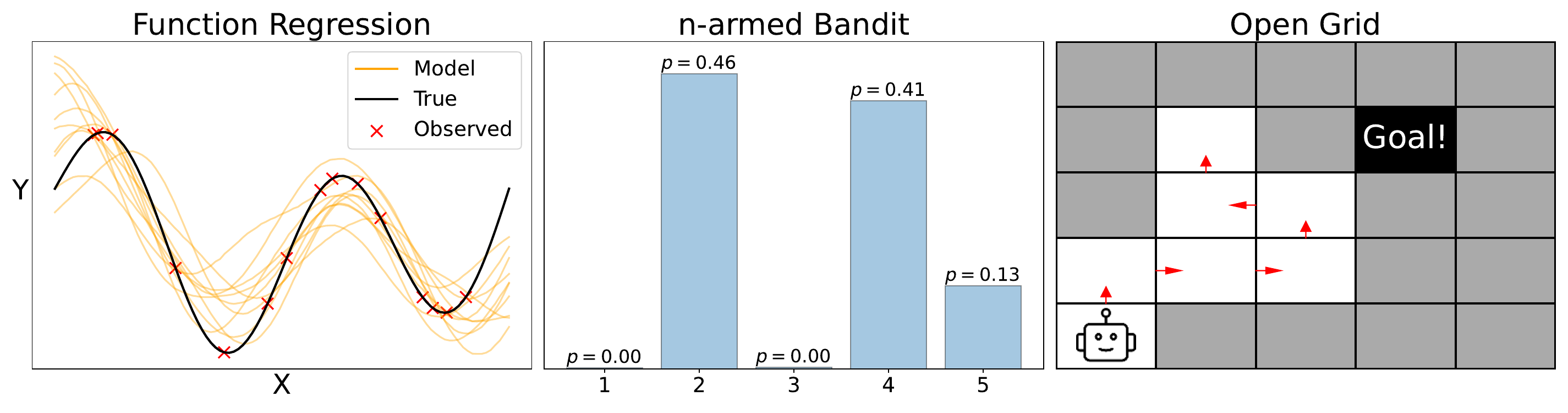}
    \caption{Visualization of sampled tasks we evaluated our method on. 1) Zero-shot learning of a function (left), 2) learning a stochastic best-arm selection algorithm (middle), and 3) learning a deterministic grid exploration agent (right).}
    \label{fig:test_env}
\end{figure*}

\textbf{Supervised.} For the supervised problem we generated noiseless $1$D test-functions on the bounded domain $[-1, 1]$. We did this through a Fourier expansion of $n=4$ components where we randomly generate amplitudes, phase shifts, and input shifts. We manually tuned the ranges for these parameters and sampled uniformly random within these ranges. In other words, we can define the joint distribution over a function dataset as,
\begin{align}
    p(X, Y) &= \delta(Y = \text{FourierSum}(Y; X, \varphi_{1:n}, c_{\text{shift}}, A_{0:n})) \cdot \text{Unif}(X; -1.0, 1.0),
    \\
    & c_{\text{shift}} \sim \text{Unif}(0.0, \pi), \varphi_i \sim \text{Unif}(0.0, \pi), A_i \sim \text{Unif}(-1.0, 1.0)
\end{align}
where the 'FourierSum' is computed in its amplitude-phase form. The training was performed with $50$ examples per sampled function.

\textbf{Bandit} To generate bandit problems we used a Dirichlet distribution with $\alpha=0.2$ during training and $\alpha = 0.3$ during testing (higher $\alpha$ makes the problem more difficult), this gave us a normalized vector of probabilities, $p_n \sim \text{Dir}(\alpha=1_n \cdot 0.2)$ for which the agent needed to find $\max_i p_i$.  Observations (which are equivalent to the rewards) were generated by sampling Bernoulli outcomes given $p_i$. Since each interaction of the agent with the bandit is seen as an episode, the environment returned discount factors of $\gamma = 0$. The training was performed over 50 total interactions.

\textbf{Gridworld} For the discrete grid environment we closely match the implementation of \cite{zintgraf_varibad_2020}. We reimplemented this environment in Jax to benefit from GPU acceleration for the data-generation process. The environment constructs a $n \times n$ open grid for the agent and uniformly randomly initializes the start and goal state (such that they don't overlap). The agent can choose between moving up, down, left, or right, the environment transitions deterministically but does not move the agent if it moves outside the bounds. The agent is rewarded with +1 if it encounters the goal and 0 otherwise, the observations were constructed as two one-hot-encoded vectors for the row and column index. The goal tile is \emph{not} observed. If the agent did not find the goal within $T=15$ steps it would be reset to its starting state and the discount factor would be set to $\gamma = 0$ at that transition. The training was performed over 100 total interactions.

\subsection{Experimental Design and Hyperparameters}

The supervised experiments did not provide additional hyperparameters to set, all necessary parameters were learned using an empirical cross-entropy with the data (see the main paper). For the reinforcement learning experiments we needed to set several hyperparameters for PPO \cite{schulman2017proximal}, these are given in Table~\ref{ap:tab:ppo_params}. We did not use mini batching for our version of recurrent PPO and accumulated the loss over the full trajectories and batches. We found that larger batch sizes gave us faster and more stable learning.

\begin{table}[h]
    \centering
    \caption{Proximal Policy Optimization loss parameters. Note that we use our Recurrent implementation for this algorithm. }
    \begin{tabular}{|l|c|c|}
        \hline
         \textbf{Name} & \textbf{Symbol} & \textbf{Value} \\
         \hline
         Minibatches    &    &  Full-Batch \\
         Batch-Size    &    &  256 \\
         TD-Lambda    &  $\lambda$   &  0.9 \\
         Discount    &  $\gamma$   &  0.9 \\
         Policy-Ratio clipping    &  $\epsilon$   &  0.2 \\
         Standardize Advantages    &    &  False \\
         Exact Policy Entropy    &    &  True \\
         Value Loss Scale    &    &  1.0 \\
         Policy Loss Scale    &    &  1.0 \\
         Entropy Loss Scale    &    &  0.1 \\
         \hline
    \end{tabular}
    \label{ap:tab:ppo_params}
\end{table}

Then our experimental design implied running an exhaustive parameter grid over the domain presented in Table~\ref{ap:tab:ablations}. This grid was adjusted over successive experiments to reduce the computational footprint, this was manually tuned to select the parameter values that performed the best for both the baseline and our method. The parameter grid was applied equivalently on all problem domains for $r=30$ distinct seeds. Although this experiment design is still quite modest, running this full grid induces $144$ distinct configurations times $30$ repetitions for the Laplace VRNN alone. One single run took on average $1 \frac12$ hour to complete for the gridworld environment on an A100 80GB NVIDIA GPU. Although, a better learning algorithm other than PPO (e.g., MPO), and minibatch optimizations could probably get the wallclock time down drastically while still achieving similar performance. 

\begin{table}[h]
    \centering
    \caption{Proximal Policy Optimization loss parameters. Note that we use our recurrent implementation for this algorithm. All configurations were repeated for $r=30$ repetitions (distinct random seeds). We also drop configurations that do not induce a valid model, e.g., $k_Z = 0$ and accumulation of $\Sigma$ is not a valid configuration since there is no window to accumulate over.}
    \begin{tabular}{|l|c|c|}
        \hline
        \textbf{Name} & \textbf{Symbol} & \textbf{Value} \\
        \hline
        \multicolumn{3}{|l|}{\textbf{Deterministic RNN}} \\
        \hline
        Posterior Dimensionality & $n$ & \{32, 64\}     \\
        \hline
        \multicolumn{3}{|l|}{\textbf{Variational RNN}} \\
        \hline
        Posterior Dimensionality & $n$ & \{32, 64\}     \\
        Covariance Parameterization &  & \{Full, Diagonal\}     \\
        Posterior KL-Penalty & $\beta$  & $\{1.0, 10^{-2}, 10^{-4}\}$     \\
        Number of Posterior Samples & $n_Z$  & $\{1, 5\}$     \\
        \hline
        \multicolumn{3}{|l|}{\textbf{Laplace VRNN}} \\
        \hline
        Posterior Dimensionality & $n$ & \{32, 64\}     \\
        Covariance Parameterization &  & \{Full, Diagonal\}     \\
        Posterior KL-Penalty & $\beta$  & $\{1.0, 10^{-2}, 10^{-4}\}$     \\
        Number of Posterior Samples & $n_Z$  & $\{1, 5\}$     \\
        History Buffer Window & $k_H$  & $\{1, 10\}$     \\
        Latent Variable Window & $k_Z$  & $\{0, 1\}$     \\
        Accumulation &  & \{$(\mu, \Sigma)$, $\Sigma$\}    \\
        \hline
    \end{tabular}
    \label{ap:tab:ablations}
\end{table}

For the finetuning experiments we essentially made model snapshots of the deterministic baselines (RNNs) and reran the ablations as shown in Table~\ref{ap:tab:ablations} using the snapshots as starting weights. For each experiment, these snapshots were taken halfway, and three-quarters way during training in terms of the number of weight-updates.

To make some final informal notes on the choices for the parameters,
\begin{itemize}
    \item We found that scaling the KL-penalties in the lower bound with hyperparameters that were slightly larger than $\beta > 10^{-2}$ caused posterior collapse \cite{goyal2017zforcing}. Meaning, our posterior simply fitted its parameters to always match the prior despite accumulating more data.
    \item For the regression problem, using multiple samples for $z^{(i)}$ inside the lower bound decreased the predictive loss a lot. Showing that integration over the predictive is effective. Although, this did not result in better test-time performance.
    \item Using a small buffer size for the history window in the posterior $q_\theta(Z_t | H_{t-k: t-1})$ is more practical since the computation of the Jacobians is the main bottleneck of our method. We found that accumulation of only the covariance where each covariance is computed with a window of just 1, $H_{t-1}$ results in the fastest method (in wallclock time) while being on par with many of the ablations.
\end{itemize}

\clearpage
\section{Supplementary Results}
\subsection{Supplementary Supervised Results}
\label{ap:sup_supervised_results}
\begin{figure}[ht]
    \centering
    \includegraphics[width=\linewidth]{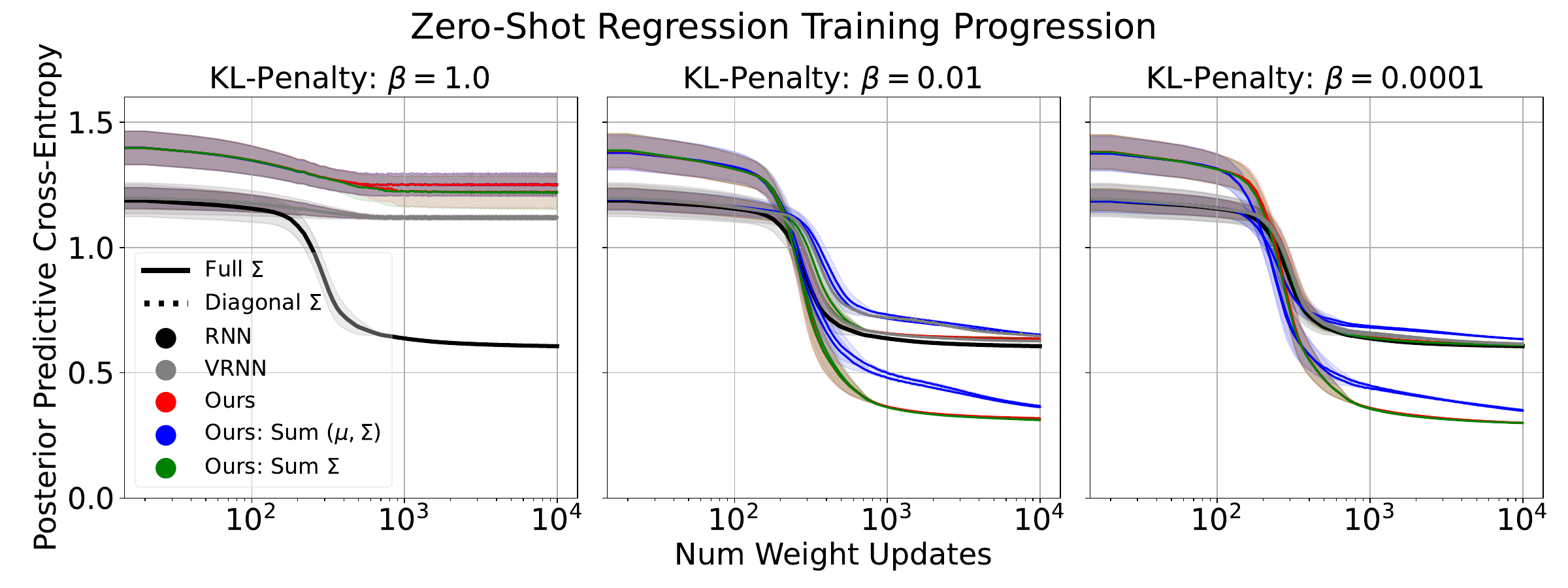}
    \caption{Progression of the predictive error during supervised model training of all ablations. Ablations are averaged over parameter groups as indicated by the legend. This figure does not show the finetuning results. To reduce computation time for subsequent results, we picked $\beta=0.01$ for the reinforcement learning task ablations.}
    \label{ap:fig:train_supervised}
\end{figure}

\begin{figure}[ht]
    \centering
    \includegraphics[width=\linewidth]{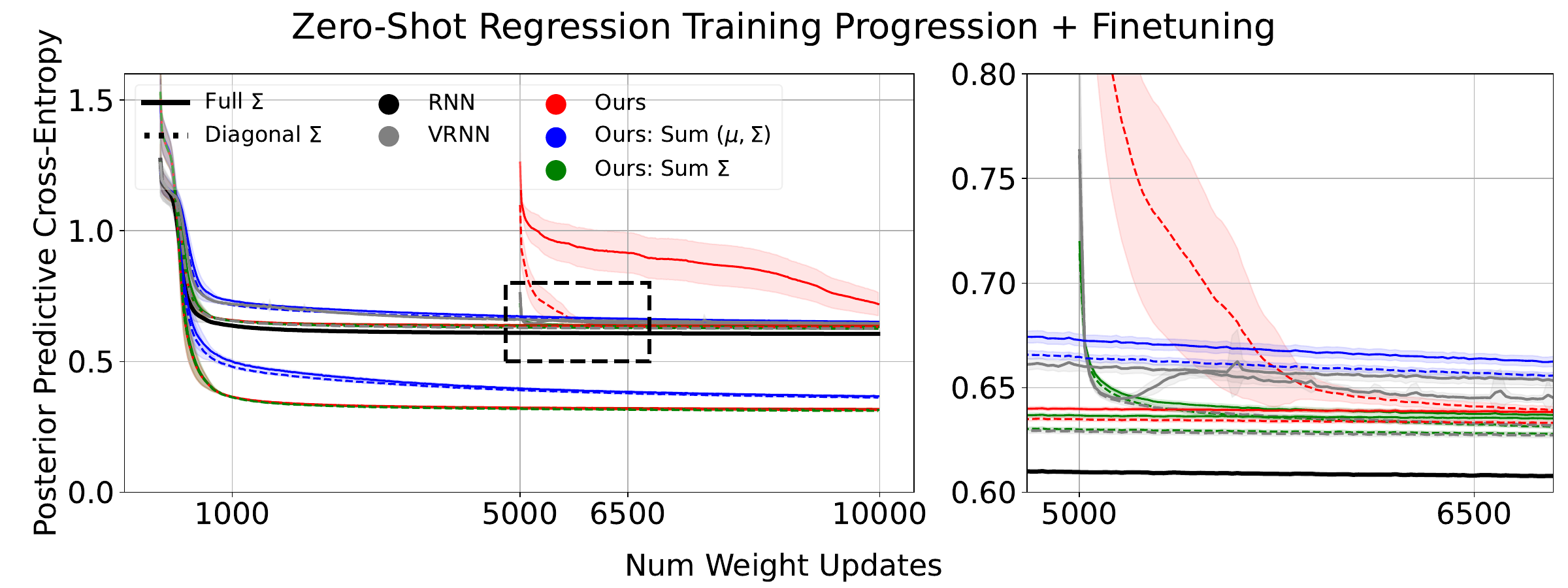}
    \caption{Zoom-in of Figure~\ref{ap:fig:train_supervised} for three finetune runs (left), for most ablations the predictive error quickly goes down to their full variational training error.}
    \label{ap:fig:train_supervised_finetune}
\end{figure}

\begin{figure}[ht]
    \centering
    \includegraphics[width=\linewidth]{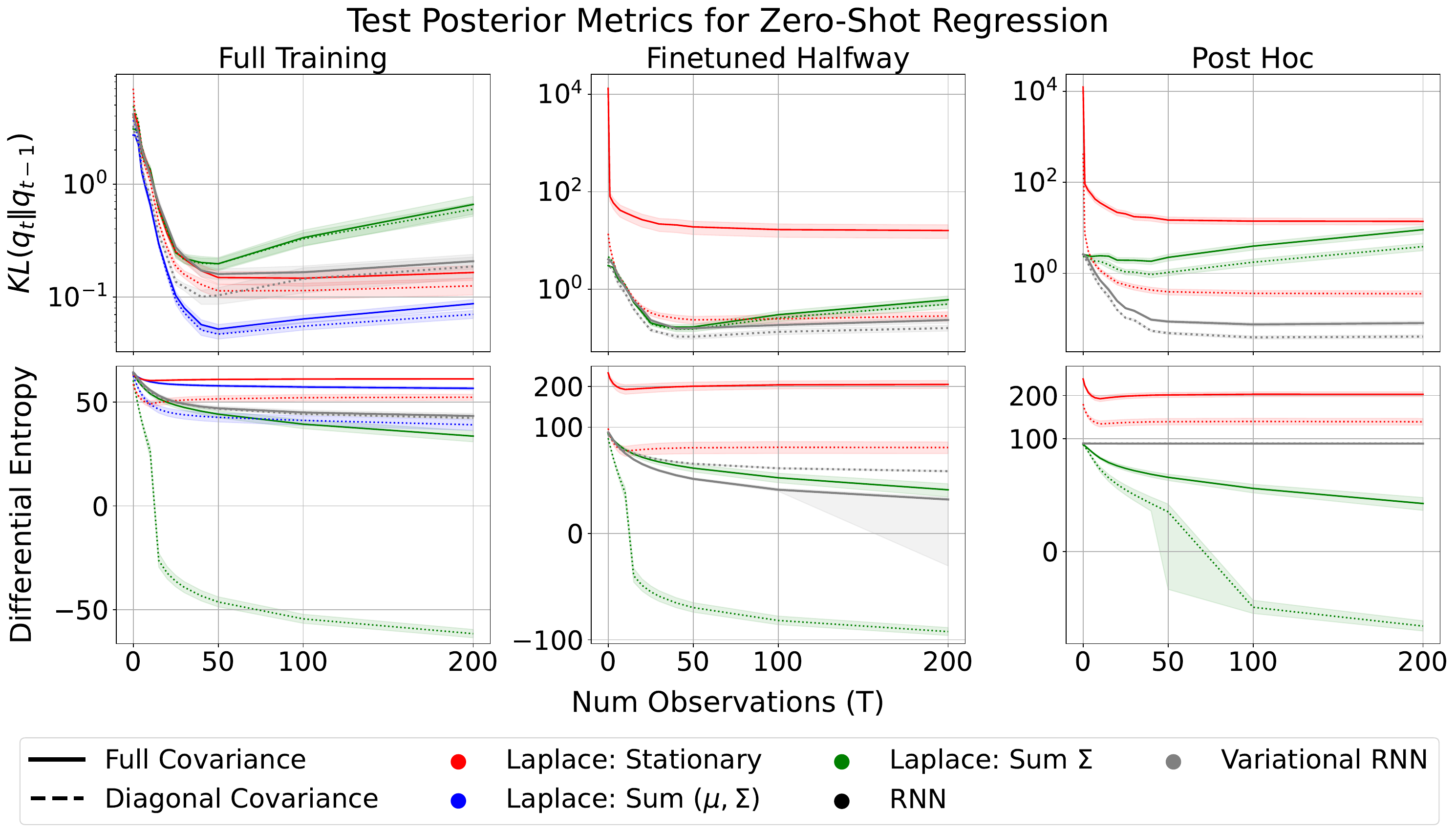}
    \caption{Posterior statistics during testing on the supervised task. The consecutive KL divergences (top) and entropy (bottom) should go down over time. The Laplace VRNN where we sum the covariances (green) is the only method that performs as expected for the entropy estimation but seems to grow more unstable for the KL-divergences. Results use $\beta=10^{-2}$.}
\end{figure}

\begin{figure}[ht]
    \centering
    \includegraphics[width=\linewidth]{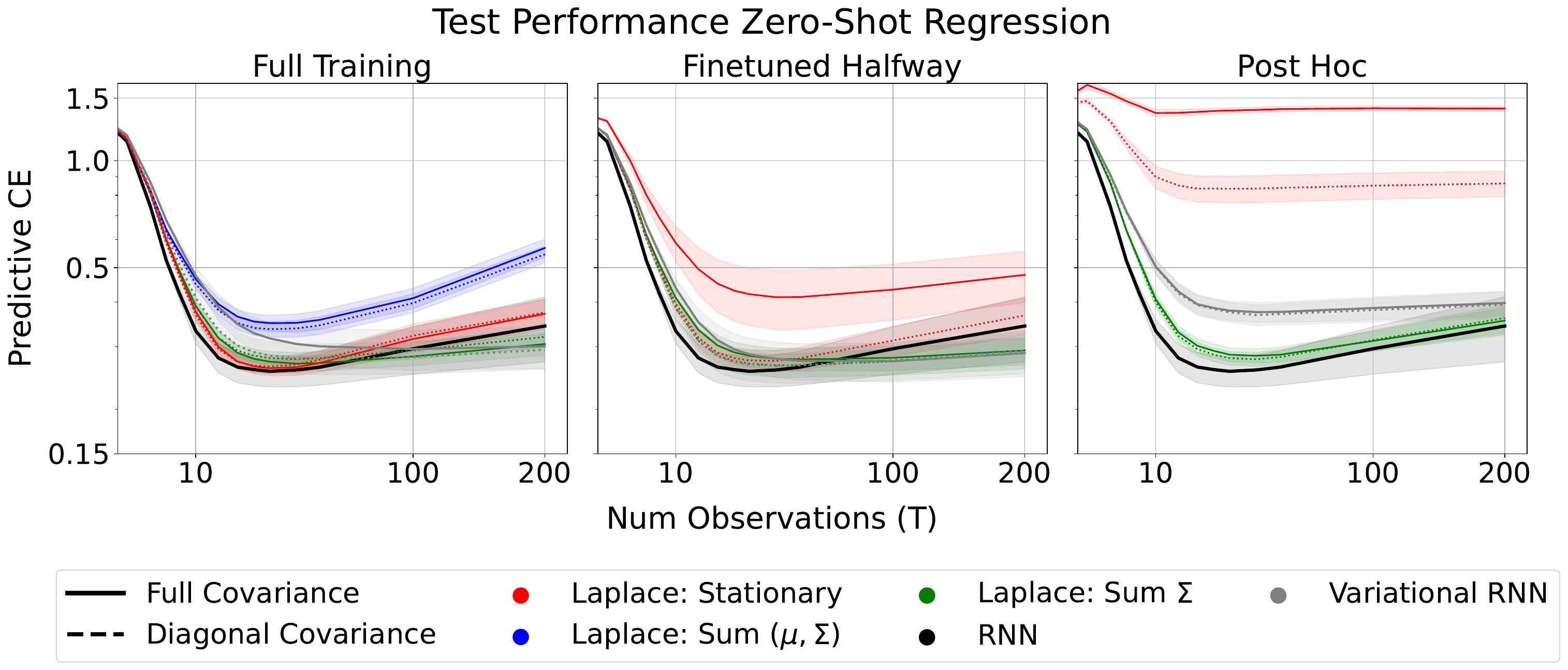}
    \caption{Complete plot of Figure~\ref{fig:supervised_test} from the main paper to include the accumulation over means and covariances (blue) in the left plot. This was left out of the main paper to improve visibility. Results use $\beta=10^{-2}$.}
\end{figure}

\clearpage
\subsection{Supplementary Bandit Results}

\begin{figure}[ht]
    \centering
    \includegraphics[width=\linewidth]{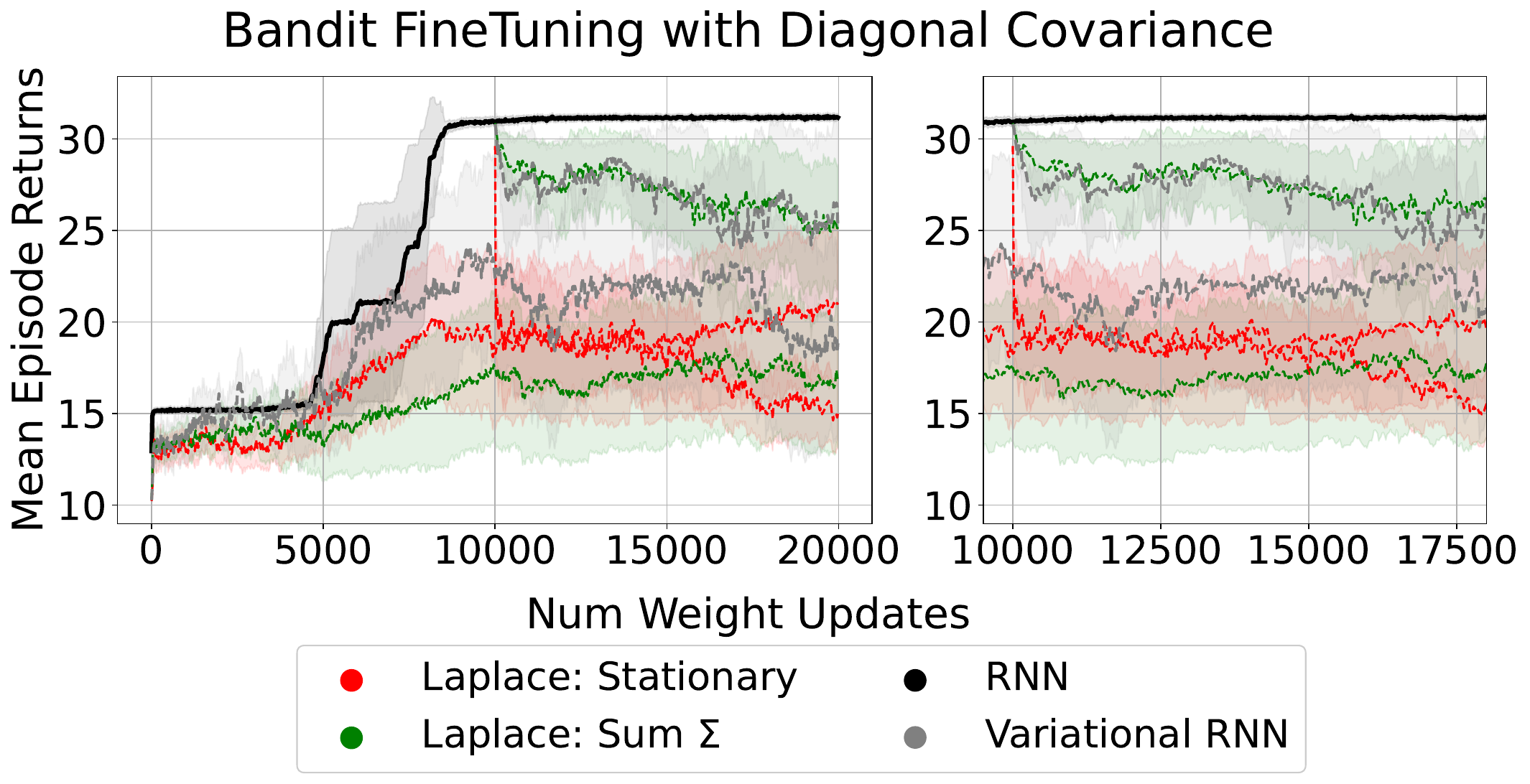}
    \caption{Zoom-in of Figure~\ref{fig:rl_train} for the bandit task when finetuning the deterministic model weights intermittently with a diagonal variational model. In this domain, it seems that finetuning slightly actually helps the expected training performance. Results use $\beta=10^{-2}$. }
\end{figure}

\begin{figure}[ht]
    \centering
    \includegraphics[width=\linewidth]{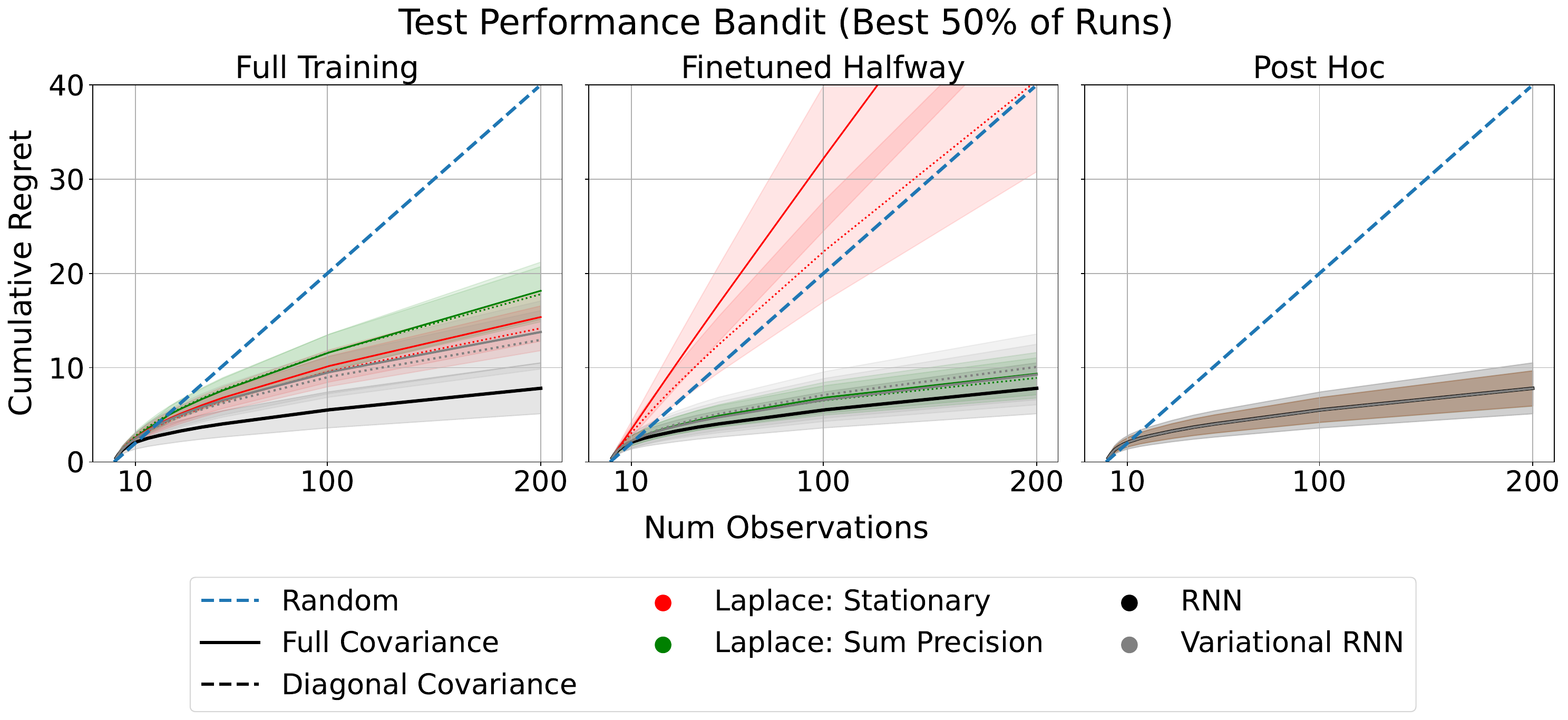}
    \caption{Cumulative regret of trained agents on the bandit task, lower is better. Since the training was quite unstable, about half of the repetitions did not find model weights with strong final performance. Only when we filtered out the better-than-median agents, did we find stronger than random performance. Results use $\beta=10^{-2}$.}
\end{figure}

\clearpage
\subsection{Supplementary Gridworld Results}
\label{ap:sup_grid_results}

\begin{figure}[ht]
    \centering
    \includegraphics[width=\linewidth]{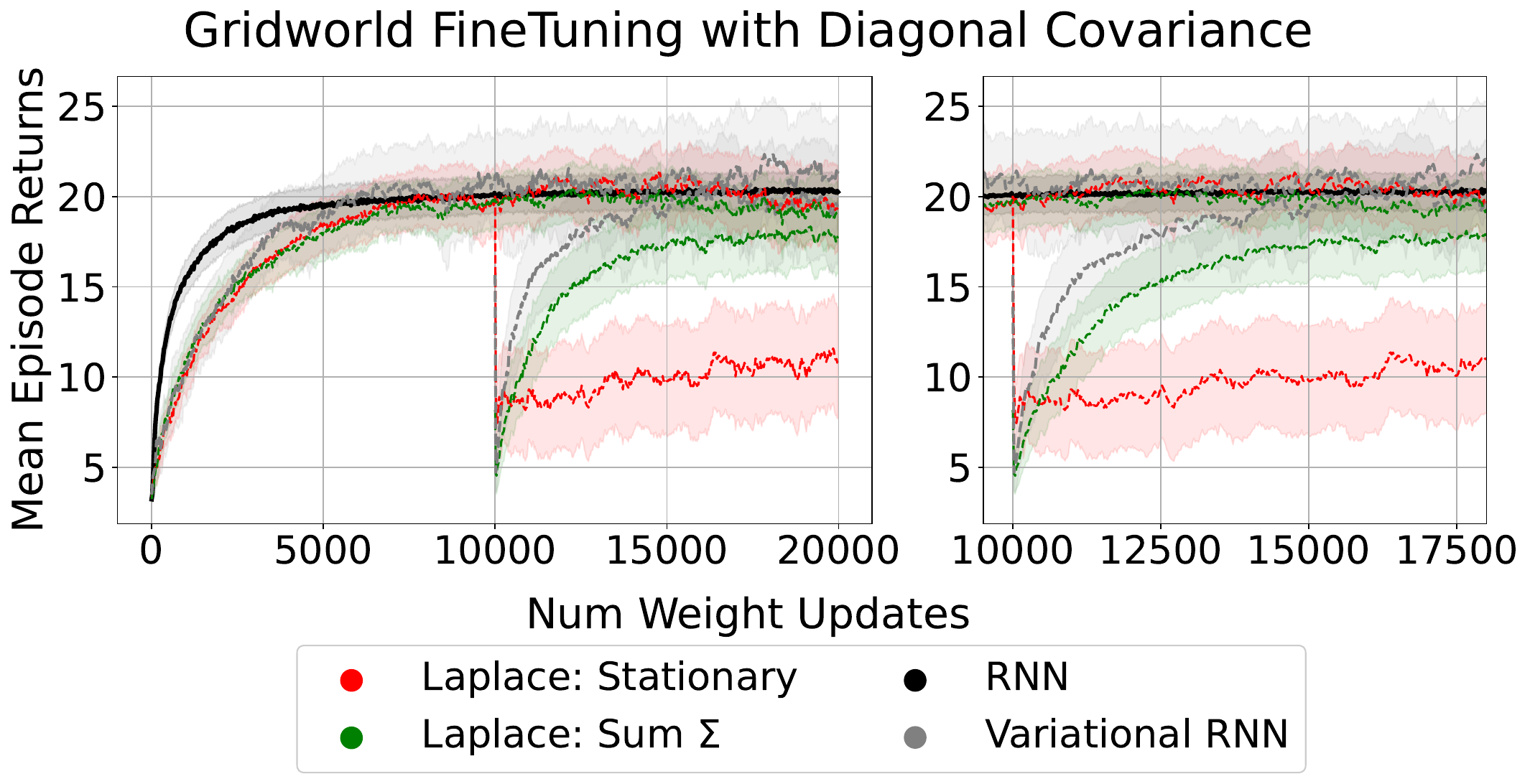}
    \caption{Zoom-in of Figure~\ref{fig:rl_train} for the gridworld task when finetuning the deterministic model weights intermittently with a diagonal variational model. In contrast to the bandit task, the agent needs to recover from this sudden change of additional model noise. Results use $\beta=10^{-2}$.}
    \label{ap:fig:rl_train_finetune}
\end{figure}

\begin{figure}[ht]
    \centering
    \includegraphics[width=\linewidth]{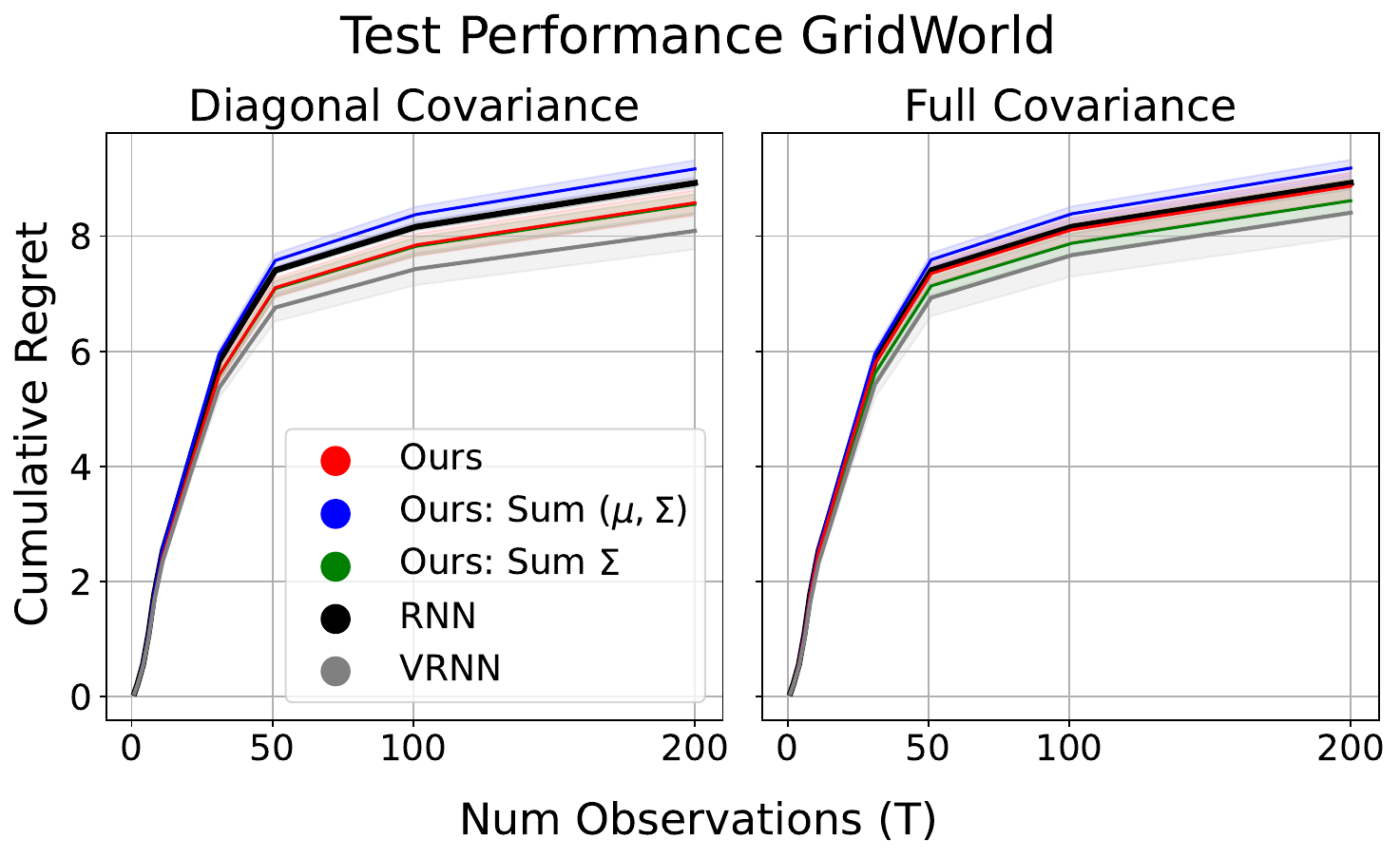}
    \caption{Cumulative regret of trained agents on the gridworld task, lower is better. All agents perform well on this task, the diagonal Variational RNN slightly outperforms all other agents. }
    \label{ap:fig:test_grid}
\end{figure}

\begin{figure}[ht]
    \centering
    \includegraphics[width=\linewidth]{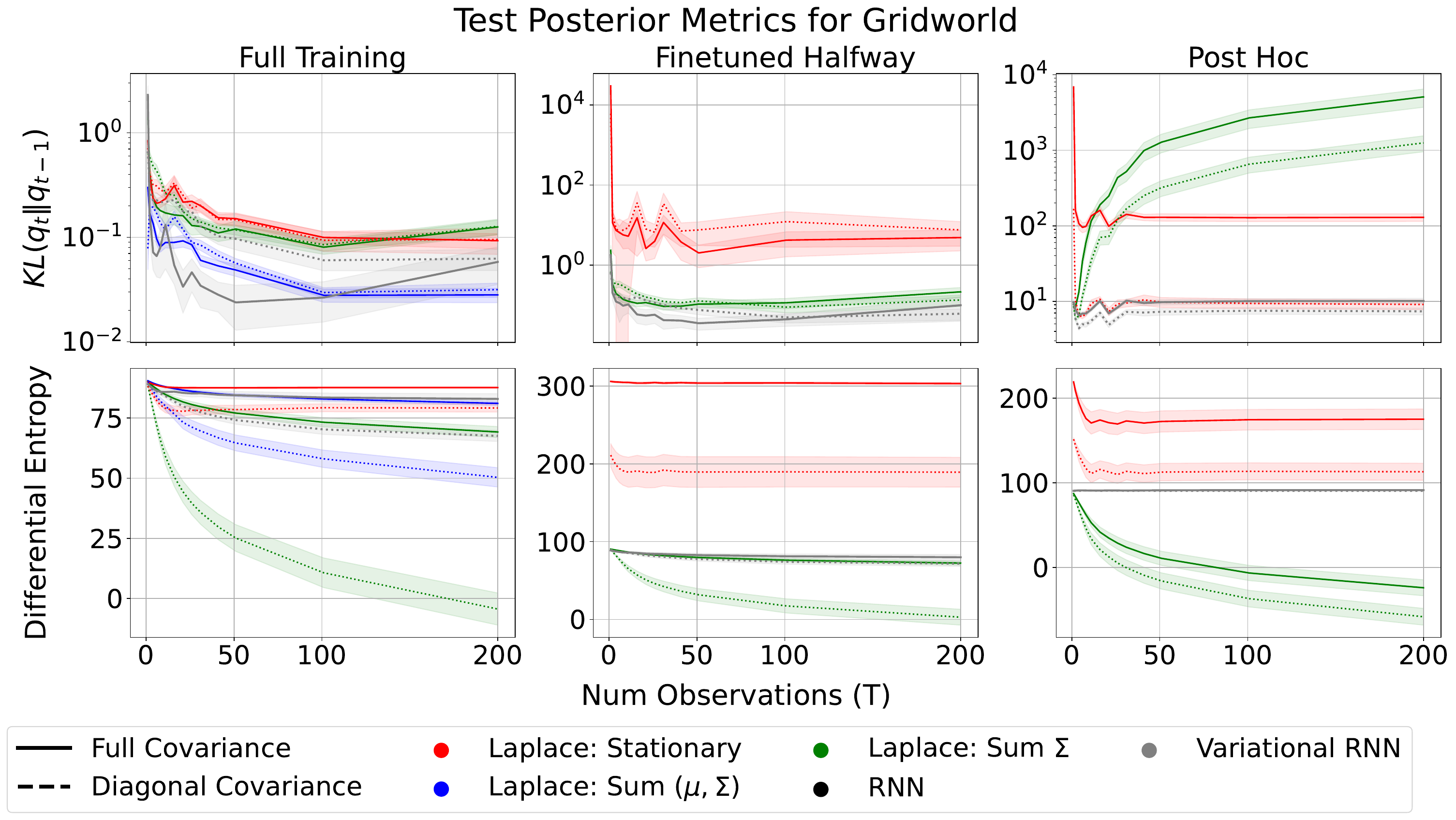}
    \caption{Complete plot of Figure~\ref{fig:grid_test_post} from the main paper to include the accumulation over mans and covariances (blue) in the left plot. Posterior statistics during testing on the gridworld task. The consecutive KL divergences (top) and entropy (bottom) should go down over time. Results use $\beta=10^{-2}$.}
\end{figure}

\end{document}